\documentclass[10pt]{article}

\usepackage[numbers]{natbib}
\usepackage{booktabs}
\usepackage{mydefs}
\usepackage{macros}

\begin{document}

\title{Bayesian Counterfactual Risk Minimization}

\author{
	Ben London \\
	\texttt{blondon@amazon.com} \\
	Amazon
	\and
	Ted Sandler \\
	\texttt{sandler@amazon.com} \\
	Amazon
}

\date{}
\maketitle



\begin{abstract}
We present a Bayesian view of counterfactual risk minimization (CRM) for offline learning from logged bandit feedback. Using PAC-Bayesian analysis, we derive a new generalization bound for the truncated inverse propensity score estimator. We apply the bound to a class of Bayesian policies, which motivates a novel, potentially data-dependent, regularization technique for CRM. Experimental results indicate that this technique outperforms standard $L_2$ regularization, and that it is competitive with variance regularization while being both simpler to implement and more computationally efficient.
\end{abstract}



\section{Introduction}
\label{sec:intro}

In industrial applications of machine learning, model development is typically an iterative process, involving multiple trials of offline training and online experimentation. For example, a content streaming service might explore various recommendation strategies in a series of A/B tests. The data that is generated by this process---e.g., impression and interaction logs---can be used to augment training data and further refine a model. However, learning from logged interactions poses two fundamental challenges: (1) the feedback obtained from interaction is always incomplete, since one only observes responses (usually referred to as \define{rewards}) for actions that were taken; (2) the distribution of observations is inherently biased by the \define{policy} that determined which action to take in each context.

This problem of learning from logged data has been studied under various names by various authors \citep{strehl:nips10,dudik:icml11,bottou:jmlr13,swaminathan:jmlr15}. We adopt the moniker \define{counterfactual risk minimization} (CRM), introduced by \citet{swaminathan:jmlr15}, though it is also known as \define{offline policy optimization} in the reinforcement learning literature. The goal of CRM is to learn a policy from data that was logged by a previous policy so as to maximize expected reward (alternatively, minimize risk) over draws of future contexts. Using an analysis based on Bennett's inequality, \citeauthor{swaminathan:jmlr15} derived an upper bound on the risk of a stochastic policy,\footnote{In a similar vein, \citet{strehl:nips10} proved a lower bound on the expected reward of a deterministic policy.} which motivated learning with variance-based regularization.

In this work, we study CRM from a Bayesian perspective, in which one's uncertainty over actions becomes uncertainty over hypotheses. We view a stochastic policy as a distribution over hypotheses, each of which is a mapping from contexts to actions. Our work bridges the gap between CRM, which has until now been approached from the frequentist perspective, and Bayesian methods, which are often used to balance exploration and exploitation in contextual bandit problems \citep{chapelle:nips11}.

Using a PAC-Bayesian analysis \citep{mcallester:colt03}, we prove an upper bound on the risk of a Bayesian policy trained on logged data. We then apply this bound to a class of Bayesian policies based on the mixed logit model. This analysis suggests an intuitive regularization strategy for Bayesian CRM based on the $L_2$ distance from the logging policy's parameters. Our \define{logging policy regularization} (LPR) is effectively similar to variance regularization, but simpler to implement and more computationally efficient. We derive two Bayesian CRM objectives based on LPR, one of which is convex. We also consider the scenario in which the logging policy is unknown. In this case, we propose a two-step procedure to learn the logging policy, and then use the learned parameters to regularize training a new policy. We prove a corresponding risk bound for this setting using a distribution-dependent prior.

We end with an empirical study of our theoretical results. First, we show that LPR outperforms standard $L_2$ regularization whenever the logging policy is better than a uniform distribution. Second, we show that LPR is competitive with variance regularization, and even outperforms it on certain problems. Finally, we demonstrate that it is indeed possible to learn the logging policy for LPR with negligible impact on performance. These findings establish LPR as a simple, effective method for Bayesian CRM.

\section{Preliminaries}
\label{sec:prelims}

Let $\Contexts$ denote a set of \define{contexts}, and $\Actions$ denote a finite set of $\NumActions$ discrete \define{actions}. We are interested in finding a \define{stochastic policy}, $\Policy : \Contexts \to \Simplex_{\NumActions}$, which maps $\Contexts$ to the probability simplex on $\NumActions$ vertices, denoted $\Simplex_{\NumActions}$; in other words, $\Policy$ defines a conditional probability distribution over actions given contexts, from which we can sample actions. For a given context, $\context \in \Contexts$, we denote the conditional distribution on $\Actions$ by $\Policy(\context)$, and the probability mass of a particular action, $\action \in \Actions$, by $\Policy(\action \| \context)$.

Each action is associated with a stochastic, contextual \define{reward}, given by an unknown function, $\Reward : \Contexts \times \Actions \to [0,1]$, which we assume is bounded. When an action is played in response to a context, we only observe the reward for said action. This type of incomplete feedback is commonly referred to as \define{bandit feedback}. We assume a stationary distribution, $\ContextDistribution$, over contexts and reward functions. Our goal will be to find a policy that maximizes the expected reward over draws of $(\context, \Reward) \by \ContextDistribution$ and $\action \by \Policy(\context)$; or, put differently, one that minimizes the \define{risk},
\begin{equation}
\Risk(\Policy) \defeq 1 - \Ep_{(\context, \Reward) \by \ContextDistribution} \Ep_{\action \by \Policy(\context)} \left[ \Reward(\context, \action) \right] .
\label{eq:risk}
\end{equation}

We assume that we have access to a dataset of logged observations (i.e., examples), $\Data \defeq ( \context_i, \action_i, \prob_i, \reward_i )_{i=1}^{\DataSize}$, where $(\context_i, \Reward)$ were sampled from $\ContextDistribution$; action $\action_i$ was sampled with probability $\prob_i \defeq \LoggingPolicy(\action_i \| \context_i)$ from a fixed \define{logging policy}, $\LoggingPolicy$; and reward $\reward_i \defeq \Reward(\context_i, \action_i)$ was observed. The distribution of $\Data$, which we denote by $(\ContextDistribution \times \LoggingPolicy)^{\DataSize}$, is biased by the logging policy, since we only observe rewards for actions that were sampled from its distribution, which may be far from uniform. However, we assume that $\LoggingPolicy$ has \define{full support}---that is, $\LoggingPolicy(\action \| \context) > 0$ for all $\context \in \Contexts$ and $\action \in \Actions$---which implies
\begin{equation}
\Ep_{\action \by \LoggingPolicy(\context)} \left[ \Reward(\context, \action) \, \frac{\Policy(\action \| \context)}{\LoggingPolicy(\action \| \context)} \right]
	= \Ep_{\action \by \Policy(\context)} \left[ \Reward(\context, \action) \right] .
\label{eq:ips_unbiased}
\end{equation}
We can therefore obtain an unbiased estimate of $\Risk(\Policy)$ by scaling each reward by its \define{inverse propensity score} (IPS) \citep{horvitz:jasa52,rosenbaum:biometrika83}, $\prob_i^{-1}$, which yields the \define{IPS estimator},
\begin{equation}
\EmpRisk(\Policy, \Data) \defeq 1 - \frac{1}{\DataSize} \sum_{i=1}^{\DataSize} \reward_i \, \frac{\Policy(\action_i \| \context_i)}{\prob_i} .
\label{eq:emp_risk}
\end{equation}
Unfortunately, without additional assumptions on $\Policy$ and $\LoggingPolicy$, IPS has unbounded variance. This issue can be mitigated by \define{truncating} (or \define{clipping}) $\prob_i$ to the interval $[\MinPS,1]$ (as proposed in \citep{strehl:nips10}), yielding
\begin{equation}
\ClipEmpRisk(\Policy, \Data) \defeq 1 - \frac{1}{\DataSize} \sum_{i=1}^{\DataSize} \reward_i \, \frac{\Policy(\action_i \| \context_i)}{\max\{ \prob_i, \MinPS \}} ,
\label{eq:clipped_emp_risk}
\end{equation}
which we will sometimes refer to as the \define{empirical risk}. This estimator has finite variance, at the cost of adding bias. Crucially, since $\max\{ \prob_i, \MinPS \} \geq \prob_i$, we have that $\ClipEmpRisk(\Policy, \Data) \geq \EmpRisk(\Policy, \Data)$, which implies
\begin{equation}
\Ep_{\Data \by (\ContextDistribution \times \LoggingPolicy)^{\DataSize}}\left[ \ClipEmpRisk(\Policy, \Data) \right]
	\geq \Ep_{\Data \by (\ContextDistribution \times \LoggingPolicy)^{\DataSize}}\left[ \EmpRisk(\Policy, \Data) \right]
	= \Risk(\Policy) .
\label{eq:clipped_emp_risk_lb_risk}
\end{equation}
Thus, if $\ClipEmpRisk(\Policy, \Data)$ concentrates around its mean, then by minimizing $\ClipEmpRisk(\dummyvar, \Data)$, we minimize a probabilistic upper bound on the risk.

\begin{remark}
\label{rem:clipping}
There are other estimators we can consider. For instance, we could truncate the ratio of the policy and the logging policy, $\min\{ \Policy(\action_i \| \context_i) / \prob_i, \MinPS^{-1} \}$ \citep{ionides:jcgs08,swaminathan:jmlr15}. However, this form of truncation is incompatible with our subsequent analysis because the policy is inside the $\min$ operator. Avoiding truncation altogether, we could use the \define{self-normalizing} estimator \citep{swaminathan:nips15}, but this is also incompatible with our analysis, since the estimator does not decompose as a sum of i.i.d.\ random variables. Finally, we note that our theory \emph{does} apply, with small modifications, to the \define{doubly-robust} estimator \citep{dudik:icml11}.
\end{remark}

\subsection{Counterfactual Risk Minimization}
\label{sec:crm}

Our work is heavily influenced by \citet{swaminathan:jmlr15}, who coined the term \define{counterfactual risk minimization} (CRM) to refer to the problem of learning a policy from logged bandit feedback by minimizing an upper bound on the risk. Their bound is a function of the truncated IPS estimator,\footnote{Though \citeauthor{swaminathan:jmlr15} used $\min\{ \Policy(\action_i \| \context_i) / \prob_i, \MinPS^{-1} \}$ truncation, their results nonetheless hold for $\max\{ \prob_i, \MinPS \}^{-1}$ truncation.} the sample variance of the truncated IPS-weighted rewards under the policy, $\SampleVar_{\MinPS}(\Policy, \Data)$, and a measure of the complexity, $\Complexity : \Policies \to \Reals_+$, of the class of policies being considered, $\Policies \subseteq \{ \Policy : \Contexts \to \Simplex_{\NumActions} \}$. Ignoring constants, their bound is of the form
\begin{equation}
\Risk(\Policy)
	\leq \ClipEmpRisk(\Policy, \Data) + \BigO\( \sqrt{\frac{\SampleVar_{\MinPS}(\Policy, \Data) \, \Complexity(\Policies)}{\DataSize}} + \frac{\Complexity(\Policies)}{\DataSize} \) .
\label{eq:crm_bound}
\end{equation}
When $\SampleVar_{\MinPS}(\Policy, \Data)$ is sufficiently small, the bound's dominating term is $\BigO(\DataSize^{-1})$, which is the so-called ``fast" learning rate. This motivates a variance-regularized learning objective,
\begin{equation}
\argmin_{\Policy \in \Policies} \, \ClipEmpRisk(\Policy, \Data) + \regparam \sqrt{\frac{\SampleVar_{\MinPS}(\Policy, \Data)}{\DataSize}} ,
\label{eq:var_reg_obj}
\end{equation}
for a regularization parameter, $\regparam > 0$. \citeauthor{swaminathan:jmlr15} propose a majorization-minimization algorithm---named \define{policy optimization for exponential models} (POEM)---to solve this optimization.

\section{PAC-Bayesian Analysis}
\label{sec:pac_bayes}

In this work, we view CRM from a Bayesian perspective. We consider stochastic policies whose action distributions are induced by distributions over \define{hypotheses}. Instead of sampling directly from a distribution on the action set, we sample from a distribution on a \define{hypothesis space}, $\Hypotheses \subseteq \{ \hypothesis : \Contexts \to \Actions \}$, in which each element is a deterministic mapping from contexts to actions.\footnote{This view of stochastic policies was also used by \citet{seldin:nips11} to analyze contextual bandits in the PAC-Bayes framework.} As such, for a distribution, $\Posterior$, on $\Hypotheses$, the probability of an action, $\action \in \Actions$, given a context, $\context \in \Contexts$, is the probability that a random hypothesis, $\hypothesis \by \Posterior$, maps $\context$ to $\action$; that is,
\begin{equation}
\Policy_{\Posterior}(\action \| \context)
	\defeq \Pr_{\hypothesis \by \Posterior}\{ \hypothesis(\context) = \action \}
	= \Ep_{\hypothesis \by \Posterior}\left[ \1\{ \hypothesis(\context) = \action \} \right] .
\label{eq:bayesian_policy}
\end{equation}
Usually, the hypothesis space consists of functions of a certain parametric form, so the distribution is actually over the parameter values. We analyze one such class in \cref{sec:mixed_logit_model}.

To analyze Bayesian policies, we use the \define{PAC-Bayesian} framework (also known as simply \define{PAC-Bayes}) \citep{mcallester:colt99,langford:nips02,seeger:jmlr02,catoni:ims07,germain:icml09}. The PAC-Bayesian learning paradigm proceeds as follows: first, we fix a hypothesis space, $\Hypotheses$, and a \define{prior} distribution, $\Prior$, on $\Hypotheses$; then, we receive some data, $\Data$, drawn according to a fixed distribution; given $\Data$, we learn a we learn a \define{posterior} distribution, $\Posterior$, on $\Hypotheses$, from which we can sample hypotheses to classify new instances. In our PAC-Bayesian formulation of CRM, the learned posterior becomes our stochastic policy (\cref{eq:bayesian_policy}). Given a context, $\context \in \Contexts$, we sample an action by sampling $\hypothesis \by \Posterior$ (independent of $\context$) and returning $\hypothesis(\context)$. (In PAC-Bayesian terminology, this procedure is often referred to as the \define{Gibbs classifier}.) 

\begin{remark}
Instead of sampling actions via a posterior over hypotheses, we could equivalently sample policies from a posterior over policies, $\{ \Policy : \Contexts \to \Simplex_{\NumActions} \}$, then sample actions from said policies. The Bayesian policy would then be the expected policy, $\bar{\Policy}_{\Posterior}(\action \| \context) \defeq \Ep_{\Policy \by \Posterior}[ \Policy(\action \| \context) ]$. That said, it is more traditional in PAC-Bayes---and perhaps more flexible---to think in terms of the Gibbs classifier, which directly maps contexts to actions.
\end{remark}

It is important to note that the choice of prior cannot depend on the training data; however, \emph{the prior can generate the data}. Indeed, we can generate $\Data$ by 
sampling $(\context_i, \Reward) \by \ContextDistribution$, $\hypothesis \by \Prior$ and logging $\( \context_i, \hypothesis(\context_i), \Policy_0(\hypothesis(\context_i) \| \context_i), \Reward(\context_i, \hypothesis(\context_i)) \)$, for $i = 1,\dots,\DataSize$. Thus, in the PAC-Bayesian formulation of CRM, \emph{the prior can be the logging policy}. We elaborate on this idea in \cref{sec:mixed_logit_model}.

\subsection{Risk Bounds}
\label{sec:pac_bayes_bounds}

The heart of our analysis is an application of the PAC-Bayesian theorem---a generalization bound for Bayesian learning---to upper-bound the risk. The particular PAC-Bayesian bound we use is by \citet{mcallester:colt03}.
\begin{lemma}
\label{lem:pac_bayes_theorem}
Let $\DataDistribution$ denote a fixed distribution on an instance space, $\DataSpace$. Let $\Loss : \Hypotheses \times \DataSpace \to [0,1]$ denote a \define{loss function}. For a distribution, $\Posterior$, on the hypothesis space, $\Hypotheses$, and a dataset, $\Data \defeq (\datum_1,\dots,\datum_{\DataSize}) \in \DataSpace^{\DataSize}$, let $\Risk(\Posterior) \defeq \Ep_{\hypothesis \by \Posterior} \Ep_{\datum \by \DataDistribution} [ \Loss(\hypothesis, \datum) ]$ and $\EmpRisk(\Posterior, \Data) \defeq \Ep_{\hypothesis \by \Posterior} \big[ \frac{1}{\DataSize} \sum_{i=1}^{\DataSize} \Loss(\hypothesis, \datum_i) \big]$ denote the risk and empirical risk, respectively. For any $\DataSize \geq 1$, $\delta \in (0,1)$, and fixed prior, $\Prior$, on $\Hypotheses$, with probability at least $1 - \delta$ over draws of $\Data \by \DataDistribution^{\DataSize}$, the following holds simultaneously for all posteriors, $\Posterior$, on $\Hypotheses$:
\begin{equation}
\Risk(\Posterior)
	\leq \EmpRisk(\Posterior, \Data)
		+ \sqrt{ \frac{ 2 \EmpRisk(\Posterior, \Data) \( \KLDiv{\Posterior}{\Prior} + \ln\frac{\DataSize}{\delta} \) }{ \DataSize - 1 } }
		+ \frac{ 2 \( \KLDiv{\Posterior}{\Prior} + \ln\frac{\DataSize}{\delta} \) }{ \DataSize - 1 } .
\label{eq:pac_bayes_theorem}
\end{equation}
\end{lemma}

The hallmark of a PAC-Bayesian bound is the KL divergence from the fixed prior to a learned posterior. This quantity can be interpreted as a complexity measure, similar to the VC dimension, covering number or Rademacher complexity \citep{mohri:book12}. The divergence penalizes posteriors that stray from the prior, effectively penalizing overfitting.

One attractive property of \citeauthor{mcallester:colt03}'s bound is that, if the empirical risk is sufficiently small, then the generalization error, $\Risk(\Posterior) - \EmpRisk(\Posterior, \Data)$, can be of order $\BigO(\DataSize^{-1})$. Thus, the bound captures both realizable and non-realizable learning problems.

We use \cref{lem:pac_bayes_theorem} to prove the following risk bound.
\begin{theorem}
\label{th:pac_bayes_bound}
Let $\Hypotheses \subseteq \{ \hypothesis : \Contexts \to \Actions \}$ denote a hypothesis space mapping contexts to actions. For any $\DataSize \geq 1$, $\delta \in (0,1)$, $\MinPS \in (0,1)$ and fixed prior, $\Prior$, on $\Hypotheses$, with probability at least $1 - \delta$ over draws of $\Data \by (\ContextDistribution \times \LoggingPolicy)^{\DataSize}$, the following holds simultaneously for all posteriors, $\Posterior$, on $\Hypotheses$:
\begin{align}
\Risk(\Policy_{\Posterior})
	\leq \ClipEmpRisk(\Policy_{\Posterior}, \Data)
		&+ \sqrt{ \frac{ 2 \big( \ClipEmpRisk(\Policy_{\Posterior}, \Data) - 1 + \frac{1}{\MinPS} \big) \( \KLDiv{\Posterior}{\Prior} + \ln\frac{\DataSize}{\delta} \) }{ \MinPS \( \DataSize - 1 \) } } \\
		&+ \frac{ 2 \( \KLDiv{\Posterior}{\Prior} + \ln\frac{\DataSize}{\delta} \) }{ \MinPS \( \DataSize - 1 \) } .
\label{eq:pac_bayes_bound}
\end{align}
\end{theorem}
\begin{proof}
To apply \cref{lem:pac_bayes_theorem}, we need to define an appropriate loss function for CRM. It should be expressed as a function of a hypothesis and a single example,\footnote{This criterion ensures that the (empirical) risk decomposes as a sum of i.i.d.\ random variables, which is our motivation for using the truncated IPS estimator over the self-normalizing estimator \citep{swaminathan:nips15}; the latter does not decompose.} and bounded in $[0,1]$. Accordingly, we define
\begin{equation}
\ClipLoss(\hypothesis, \context, \action, \prob, \reward)
	\defeq 1 - \MinPS \, \reward \, \frac{ \1\{ \hypothesis(\context) = \action \} }{ \max\{ \prob, \MinPS \} } ,
\label{eq:loss}
\end{equation}
which satisfies these criteria. Using this loss function, we let
\begin{align}
\label{eq:gibbs_risk}
&\ClipRisk(\Posterior) \defeq \Ep_{\hypothesis \by \Posterior} \Ep_{(\context, \Reward) \by \ContextDistribution} \Ep_{\action \by \LoggingPolicy(\context)} [ \ClipLoss( \hypothesis, \context, \action, \LoggingPolicy(\action \| \context), \Reward(\context, \action) ) ] \\
\eqand
\label{eq:gibbs_emp_risk}
&\ClipEmpRisk(\Posterior, \Data) \defeq \Ep_{\hypothesis \by \Posterior} \left[ \frac{1}{\DataSize} \sum_{i=1}^{\DataSize} \ClipLoss(\hypothesis, \context_i, \action_i, \prob_i, \reward_i) \right] .
\end{align}
Importantly, $\ClipEmpRisk(\Posterior, \Data)$ is an unbiased estimate of $\ClipRisk(\Posterior)$,
\begin{equation}
\Ep_{\Data \by (\ContextDistribution \times \LoggingPolicy)^{\DataSize}}[ \ClipEmpRisk(\Posterior, \Data) ] = \ClipRisk(\Posterior) ,
\label{eq:gibbs_emp_risk_unbiased}
\end{equation}
and a draw of $\hypothesis \by \Posterior$ does not depend on context, so $\ClipRisk(\Posterior)$ and $\ClipEmpRisk(\Posterior, \Data)$ can be expressed as expectations over $\hypothesis \by \Posterior$.\footnote{This is why we truncate with $\max\{ \prob_i, \MinPS \}^{-1}$ instead of $\min\{ \Policy(\action_i \| \context_i) / \prob_i, \MinPS^{-1} \}$.} Further, via linearity of expectation,
\begin{align}
\ClipRisk(\Posterior)
	&= 1 - \MinPS \Ep_{(\context, \Reward) \by \ContextDistribution} \Ep_{\action \by \LoggingPolicy(\context)} \left[ \Reward(\context, \action) \, \frac{ \Ep_{\hypothesis \by \Posterior}\left[ \1\{ \hypothesis(\context) = \action \} \right] }{ \max\{ \LoggingPolicy(\action \| \context), \MinPS \} } \right] \\
	&= 1 - \MinPS \Ep_{(\context, \Reward) \by \ContextDistribution} \Ep_{\action \by \LoggingPolicy(\context)} \left[ \Reward(\context, \action) \, \frac{ \Policy_{\Posterior}(\action \| \context) }{ \max\{ \LoggingPolicy(\action \| \context), \MinPS \} } \right] \\
	&\geq 1 - \MinPS \Ep_{(\context, \Reward) \by \ContextDistribution} \Ep_{\action \by \Policy_{\Posterior}(\context)} \left[ \Reward(\context, \action) \right] \\
	&= 1 - \MinPS \( 1 - \Risk(\Policy_{\Posterior}) \) ,
\label{eq:gibbs_risk_lb}
\end{align}
and
\begin{align}
\ClipEmpRisk(\Posterior, \Data)
	&= 1 - \frac{\MinPS}{\DataSize} \sum_{i=1}^{\DataSize} \reward_i \, \frac{ \Ep_{\hypothesis \by \Posterior} \left[ \1\{ \hypothesis(\context_i) = \action_i \} \right] }{ \max\{ \prob_i, \MinPS \} } \\
	&= 1 - \frac{\MinPS}{\DataSize} \sum_{i=1}^{\DataSize} \reward_i \, \frac{ \Policy_{\Posterior}(\action_i \| \context_i) }{ \max\{ \prob_i, \MinPS \} } \\
	&= 1 - \MinPS \( 1 - \ClipEmpRisk(\Policy_{\Posterior}, \Data) \) .
\label{eq:gibbs_emp_risk_equiv}
\end{align}
Thus,
\begin{equation}
\ClipRisk(\Posterior) - \ClipEmpRisk(\Posterior, \Data)
	\geq \MinPS \big( \Risk(\Policy_{\Posterior}) - \ClipEmpRisk(\Policy_{\Posterior}, \Data) \big) ,
\label{eq:gibbs_generr_lb}
\end{equation}
which means that \cref{lem:pac_bayes_theorem} can be used to upper-bound $\Risk(\Policy_{\Posterior}) - \ClipEmpRisk(\Policy_{\Posterior}, \Data)$.
\end{proof}

It is important to note that the truncated IPS estimator, $\ClipEmpRisk$, can be negative, achieving its minimum at $1 - \MinPS^{-1}$. This means that when $\ClipEmpRisk$ is minimized, the middle $\BigO(\DataSize^{-1/2})$ term disappears and the $\BigO(\DataSize^{-1})$ term dominates the bound, yielding the ``fast" learning rate. That said, our bound may not be as tight as \citeauthor{swaminathan:jmlr15}' (\cref{eq:crm_bound}), since the variance is sometimes smaller than the mean. To achieve a similar rate, we could perhaps use a PAC-Bayesian Bernstein inequality \citep{seldin:it12,tolstikhin:nips13}.

\cref{th:pac_bayes_bound} assumes that the truncation parameter, $\MinPS$, is fixed \emph{a priori}. However, using a covering technique, we can derive a risk bound that holds for all $\MinPS$ simultaneously---meaning, $\MinPS$ can be data-dependent, such as the $10\nth$ percentile of the logged propensities.
\begin{theorem}
\label{th:pac_bayes_bound_allclipping}
Let $\Hypotheses \subseteq \{ \hypothesis : \Contexts \to \Actions \}$ denote a hypothesis space mapping contexts to actions. For any $\DataSize \geq 1$, $\delta \in (0,1)$ and fixed prior, $\Prior$, on $\Hypotheses$, with probability at least $1 - \delta$ over draws of $\Data \by (\ContextDistribution \times \LoggingPolicy)^{\DataSize}$, the following holds simultaneously for all posteriors, $\Posterior$, on $\Hypotheses$, and all $\MinPS \in (0,1)$:
\begin{align}
\Risk(\Policy_{\Posterior})
	\leq \ClipEmpRisk(\Policy_{\Posterior}, \Data)
		&+ \sqrt{ \frac{ 4 \big( \ClipEmpRisk(\Policy_{\Posterior}, \Data) - 1 + \frac{2}{\MinPS} \big) \( \KLDiv{\Posterior}{\Prior} + \ln\frac{2 \DataSize}{\delta \MinPS} \) }{ \MinPS \( \DataSize - 1 \) } } \\
		&+ \frac{ 4 \( \KLDiv{\Posterior}{\Prior} + \ln\frac{2 \DataSize}{\delta \MinPS} \) }{ \MinPS \( \DataSize - 1 \) } .
\label{eq:pac_bayes_bound_allclipping}
\end{align}
\end{theorem}
\noindent
We defer the proof to \cref{sec:proof_pac_bayes_bound_allclipping}.

\cref{th:pac_bayes_bound,th:pac_bayes_bound_allclipping} hold for any fixed prior, but they have an intriguing interpretation when the prior is defined as the logging policy. In this case, one can minimize an upper bound on the risk by minimizing the empirical risk while keeping the learned policy close to the logging policy. We explore this idea, and its relationship to variance regularization and trust region methods, in the next section.

\section{Mixed Logit Models}
\label{sec:mixed_logit_model}

We will apply our PAC-Bayesian analysis to the following class of stochastic policies. We first define a hypothesis space,
\begin{equation}
\Hypotheses \defeq \{ \hypothesis_{\params, \gumbels} : \params \in \Reals^{\FeatDim}, \gumbels \in \Reals^{\NumActions} \} ,
\label{eq:mixed_logit_class}
\end{equation}
of functions of the form
\begin{equation}
\label{eq:sampler}
\hypothesis_{\params, \gumbels}(\context) \defeq \argmax_{\action \in \Actions} \, \params \cdot \Feats(\context, \action) + \gumbels_{\action} ,
\end{equation}
where $\Feats(\context, \action) \in \Reals^{\FeatDim}$ outputs features of the context and action, subject to a boundedness constraint, $\sup_{\context \in \Contexts, \action \in \Actions} \norm{\Feats(\context, \action)} \leq \FeatMaxNorm$. If each $\gumbels_{\action}$ is sampled from a \define{standard Gumbel} distribution, $\Gumbel(0,1)$ (location $0$, scale $1$), then $\hypothesis_{\params, \gumbels}(\context)$ produces a sample from a \define{softmax} policy,
\begin{equation}
\SoftmaxPolicy_{\params}(\action \| \context)
	\defeq \frac{ \exp( \params \cdot \Feats(\context, \action) ) }{ \sum_{\action' \in \Actions} \exp( \params \cdot \Feats(\context, \action') ) }
	= \Ep_{\gumbels \by \Gumbel(0,1)^{\NumActions}} \left[ \1\{ \hypothesis_{\params, \gumbels}(\context) = \action \} \right] .
\label{eq:softmax_policy}
\end{equation}
Further, if $\params$ is normally distributed, then $\hypothesis_{\params, \gumbels}(\context)$ has a \define{logistic-normal} distribution \citep{aitchison:biometrika80}.

We define the posterior, $\Posterior$, as a Gaussian over softmax parameters, $\params \by \Norm(\LearnedParams, \variance \mat{I})$, for some learned $\LearnedParams \in \Reals^{\FeatDim}$ and $\variance \in (0,\infty)$, with standard Gumbel perturbations, $\gumbels \by \Gumbel(0,1)^{\NumActions}$. As such, we have that
\begin{equation}
\Policy_{\Posterior}(\action \| \context)
	= \Ep_{(\params, \gumbels) \by \Posterior} \left[ \1\{ \hypothesis_{\params, \gumbels}(\context) = \action \} \right]
	= \Ep_{\params \by \Norm(\LearnedParams, \variance \mat{I})} \left[ \SoftmaxPolicy_{\params}(\action \| \context) \right] .
\label{eq:mixed_logit_policy}
\end{equation}
This model is alternately referred to as a \define{mixed logit} or \define{random parameter logit}.

We can define the prior in any way that seems reasonable---without access to training data, of course. In the absence of any prior knowledge, a logical choice of prior is the standard (zero mean, unit variance) multivariate normal distribution, with standard Gumbel perturbations. This prior corresponds to a Bayesian policy that takes uniformly random actions, and motivates standard $L_2$ regularization of $\LearnedParams$. However, we know that the data was generated by the logging policy, and this knowledge motivates a different kind of prior (hence, regularizer). If the logging policy performs better than a uniform action distribution---which we can verify empirically, using IPS reward estimation with the logs---then it makes sense to define the prior in terms of the logging policy.

Let us assume that the logging policy is known (we relax this assumption in \cref{sec:unknown_logging_policy}) and has a softmax form (\cref{eq:softmax_policy}), with parameters $\LoggingParams \in \Reals^{\FeatDim}$. We define the prior, $\Prior$, as an isotropic Gaussian centered at the logging policy's parameters, $\params \by \Norm(\LoggingParams, \variance_0 \mat{I})$, for some predetermined $\variance_0 \in (0,\infty)$, with standard Gumbel perturbations, $\gumbels \by \Gumbel(0,1)^{\NumActions}$. This prior encodes a belief that the logging policy, while not perfect, is a good starting point. Using the logging policy to define the prior does not violate the PAC-Bayes paradigm, since the logging policy is fixed before generating the training data. The Bayesian policy induced by this prior may not correspond to the actual logging policy, but we can define the prior any way we want.

\begin{remark}
\label{rem:gaussian}
We used isotropic covariances for the prior and posterior in order to simplify our analysis and presentation. That said, it is possible to use more complex covariance structures.
\end{remark}

\subsection{Bounding the KL Divergence}
\label{sec:bounding_kl}

The KL divergence between the above prior and posterior constructions motivates an interesting regularizer for CRM. To derive it, we upper-bound the KL divergence by a function of the model parameters.
\begin{lemma}
\label{lem:kl_div_bound}
For distributions $\Prior \defeq \Norm(\LoggingParams, \variance_0 \mat{I}) \times \Gumbel(0,1)^{\NumActions}$ and $\Posterior \defeq \Norm(\LearnedParams, \variance \mat{I}) \times \Gumbel(0,1)^{\NumActions}$, with $\LoggingParams, \LearnedParams \in \Reals^{\FeatDim}$ and $0 < \variance \leq \variance_0 < \infty$,
\begin{equation}
\KLDiv{\Posterior}{\Prior} \leq \frac{\norm{\LearnedParams - \LoggingParams}^2}{2 \variance_0} + \frac{\FeatDim}{2} \ln\frac{\variance_0}{\variance} .
\label{eq:kl_div_bound}
\end{equation}
\end{lemma}
\begin{proof}
We can ignore the Gumbel distributions, since they are identical. Using the definition of the KL divergence for multivariate Gaussians, and properties of diagonal matrices (since both covariances are diagonal), we have that
\begin{equation}
\KLDiv{\Posterior}{\Prior} = \frac{\norm{\LearnedParams - \LoggingParams}^2}{2 \variance_0} + \frac{\FeatDim}{2} \( \ln\frac{\variance_0}{\variance} + \frac{\variance}{\variance_0} - 1\) .
\label{eq:kl_div_isotropic_gaussians}
\end{equation}
We conclude by noting that $\frac{\variance}{\variance_0} - 1 \leq 0$ for $\variance \leq \variance_0$.
\end{proof}

One implication of \cref{lem:kl_div_bound}, captured by the term $\norm{\LearnedParams - \LoggingParams}^2$, is that, to generalize, the learned policy's parameters should stay close to the logging policy's parameters. This intuition concurs with \citepos{swaminathan:jmlr15} variance regularization, since one way to reduce the variance is to not stray too far from the logging policy.\footnote{Staying close to the logging policy is not the only way to reduce variance. Indeed, a policy could simply avoid the logged actions and achieve zero variance---albeit at the cost of earning zero reward.} Implementing \cref{lem:kl_div_bound}'s guideline in practice requires a simple modification to the usual $L_2$ regularization: instead of $\regparam \norm{\LearnedParams}^2$ (where $\regparam > 0$ controls the amount of regularization), use $\regparam \norm{\LearnedParams - \LoggingParams}^2$. We will henceforth refer to this as \define{logging policy regularization}. For now, we will assume that the logging policy's parameters, $\LoggingParams$, are known; we address the scenario in which the logging policy is unknown in \cref{sec:unknown_logging_policy}.

\begin{remark}
\label{rem:trpo}
The idea of regularizing by the logging policy is reminiscent of \define{trust region policy optimization} (TRPO) \citep{schulman:icml15}, a reinforcement learning algorithm in which each update to the policy's action distribution is constrained to not diverge too much from the current distribution. TRPO can be formulated as a regularizer, $\max_{\context\in\Contexts} \KLDiv{\LoggingPolicy(\context)}{\Policy(\context)}$, with $\context$ denoting the \define{state} in reinforcement learning. Interestingly, when both policies are from the softmax family, with respective parameters $\params_0$ and $\params$, one obtains an upper bound,
\begin{equation}
\max_{\context\in\Contexts} \, \KLDiv{\SoftmaxPolicy_{\params_0}(\context)}{\SoftmaxPolicy_{\params}(\context)}
	\leq 2 \FeatMaxNorm \norm{\params - \params_0} .
\label{eq:trpo_equivalence}
\end{equation}
The inequality follows from Fenchel duality and Cauchy-Schwarz. \cref{eq:trpo_equivalence} looks remarkably similar to \cref{eq:kl_div_bound} in that it involves the distance from the learned parameters, $\params$, to the logging policy's parameters, $\params_0$. Thus, for softmax policies, logging policy regularization is effectively like TRPO---but easier to compute if $\FeatDim \ll \NumActions$.
\end{remark}

Another implication of \cref{lem:kl_div_bound} is that the variance parameters of the prior and posterior---$\variance_0$ and $\variance$, respectively---affect the KL divergence, which can be thought of as the variance of the risk estimator. As we show in \cref{sec:approx_action_probs}, $\variance$ can also affect the bias of the risk estimator. Thus, selecting these parameters controls the bias-variance trade-off. We discuss this trade-off in \cref{sec:mixed_logit_bcrm}.

\subsection{Approximating the Action Probabilities}
\label{sec:approx_action_probs}

In practice, computing the posterior action probabilities (\cref{eq:mixed_logit_policy}) of a mixed logit model is difficult, since there is no analytical expression for the mean of the logistic-normal distribution \citep{aitchison:biometrika80}. It is therefore difficult to log propensities, or to compute the IPS estimator, which is a function of the learned and logged probabilities. Since it is easy to sample from a mixed logit, we can use Monte Carlo methods to estimate the probabilities. Alternatively, we can bound the probabilities by a function of the mean parameters, $\LearnedParams$.

\begin{lemma}
\label{lem:mixed_logit_mean_bounds}
If $\sup_{\context \in \Contexts, \action \in \Actions} \norm{\Feats(\context, \action)} \leq \FeatMaxNorm$, then
\begin{equation}
\SoftmaxPolicy_{\LearnedParams}(\action \| \context) \exp( -\tfrac{1}{2} \variance \FeatMaxNorm^2 )
\,\leq\,
\Policy_{\Posterior}(\action \| \context)
\,\leq\,
\SoftmaxPolicy_{\LearnedParams}(\action \| \context) \exp( 2 \variance \FeatMaxNorm^2 ) .
\label{eq:mixed_logit_mean_bounds}
\end{equation}
\end{lemma}
\noindent
We defer the proof to \cref{sec:proof_mixed_logit_mean_bounds}.

By \cref{lem:mixed_logit_mean_bounds}, the softmax probabilities induced by the mean parameters provide lower and upper bounds on the probabilities of the mixed logit model. The bounds tighten as the variance, $\variance$, becomes smaller. For instance, if $\variance = \BigO(\DataSize^{-1})$, then $\Policy_{\Posterior}(\action \| \context) \to \SoftmaxPolicy_{\LearnedParams}(\action \| \context)$ as $\DataSize \to \infty$.

During learning, we can use the lower bound of the learned probabilities to upper-bound the IPS estimator. We overload our previous notation to define a new estimator,
\begin{equation}
\ClipEmpRisk(\LearnedParams, \variance, \Data) \defeq 1 - \frac{\exp(-\frac{1}{2} \variance \FeatMaxNorm^2)}{\DataSize} \sum_{i=1}^{\DataSize} \reward_i \, \frac{\SoftmaxPolicy_{\LearnedParams}(\action_i \| \context_i)}{\max\{ \prob_i, \MinPS \}} .
\label{eq:mean_param_emp_risk}
\end{equation}
This estimator is biased, but the bias decreases with $\variance$. Importantly, $\ClipEmpRisk(\LearnedParams, \variance, \Data)$ is easy to compute, since it avoids the logistic-normal integral.\footnote{This statement assumes that the action set is not too large to compute the normalizing constant of the action distribution.}

When the learned posterior is deployed, we can log the upper bound of the propensities, so that future training with the logged data has an upper bound on the IPS estimator.

\subsection{Bayesian CRM for Mixed Logit Models}
\label{sec:mixed_logit_bcrm}

We now present a risk bound for the Bayesian policy, $\Policy_{\Posterior}$, in terms of the softmax policy, $\SoftmaxPolicy_{\LearnedParams}$, given by the mean parameters, $\LearnedParams$. Though stated for fixed $\MinPS$ using \cref{th:pac_bayes_bound}, one can easily derive an analogous bound for data-dependent $\MinPS$ using \cref{th:pac_bayes_bound_allclipping}.
\begin{theorem}
\label{th:mixed_logit_risk_bound}
Let $\Hypotheses$ denote the hypothesis space defined in \cref{eq:mixed_logit_class,eq:sampler}, and let $\Policy_{\Posterior}$ denote the mixed logit policy defined in \cref{eq:mixed_logit_policy}. For any $\DataSize \geq 1$, $\delta \in (0,1)$, $\MinPS \in (0,1)$, $\LoggingParams \in \Reals^{\FeatDim}$ and $\variance_0 \in (0,\infty)$, with probability at least $1 - \delta$ over draws of $\Data \by (\ContextDistribution \times \LoggingPolicy)^{\DataSize}$, the following holds simultaneously for all $\LearnedParams \in \Reals^{\FeatDim}$ and $\variance \in (0,\variance_0]$:
\begin{align}
\Risk(\Policy_{\Posterior})
	\leq \ClipEmpRisk(\LearnedParams, \variance, \Data)
	&+ \sqrt{ \frac{ \big( \ClipEmpRisk(\LearnedParams, \variance, \Data) - 1 + \frac{1}{\MinPS} \big) \big( \KLBound(\LoggingParams, \variance_0, \LearnedParams, \variance) + 2 \ln\frac{\DataSize}{\delta} \big) }{ \MinPS \( \DataSize - 1 \) } } \\
	&+ \frac{ \KLBound(\LoggingParams, \variance_0, \LearnedParams, \variance) + 2 \ln\frac{\DataSize}{\delta} }{ \MinPS \( \DataSize - 1 \) } ,
\label{eq:mixed_logit_risk_bound}
\end{align}
\begin{equation}
\text{where}\quad
\KLBound(\LoggingParams, \variance_0, \LearnedParams, \variance)
	\defeq \frac{\norm{\LearnedParams - \LoggingParams}^2}{\variance_0} + \FeatDim \ln\frac{\variance_0}{\variance} .
\label{eq:mixed_logit_kl_bound}
\end{equation}
\end{theorem}
\begin{proof}
Using \cref{lem:mixed_logit_mean_bounds}, it is easy to show that $\ClipEmpRisk(\Policy_{\Posterior}, \Data) \leq \ClipEmpRisk(\LearnedParams, \variance, \Data)$. The rest of the proof follows from using \cref{lem:kl_div_bound} to upper-bound the KL divergence in \cref{th:pac_bayes_bound}.
\end{proof}

\cref{th:mixed_logit_risk_bound} provides an upper bound on the risk that can be computed with training data. Moreover, the bound is differentiable and smooth, meaning it can be optimized using gradient-based methods. This motivates a new regularized learning objective for Bayesian CRM.
\begin{proposition}
\label{prop:mixed_logit_bcrm_objective}
The following optimization minimizes an upper bound on \cref{eq:mixed_logit_risk_bound}:
\begin{equation}
\argmin_{\substack{\LearnedParams \in \Reals^{\FeatDim} \\ \variance \in (0, \variance_0]}} \,
	\ClipEmpRisk(\LearnedParams, \variance, \Data)
	+ \frac{\norm{\LearnedParams - \LoggingParams}^2}{\variance_0 \, \MinPS \( \DataSize - 1 \)}
	- \frac{\FeatDim \ln\variance}{\MinPS \( \DataSize - 1 \)} .
\label{eq:mixed_logit_bcrm_objective}
\end{equation}
\end{proposition}
\cref{eq:mixed_logit_bcrm_objective} is unfortunately non-convex. However, we can upper-bound $\ClipEmpRisk(\LearnedParams, \variance, \Data)$ to obtain an objective that is differentiable, smooth and \emph{convex}.
\begin{proposition}
\label{prop:mixed_logit_bcrm_convex_objective}
The following convex optimization minimizes an upper bound on \cref{eq:mixed_logit_risk_bound}:
\begin{equation}
\argmin_{\LearnedParams \in \Reals^{\FeatDim}} \,
	\frac{1}{\DataSize} \sum_{i=1}^{\DataSize} - \frac{\reward_i \ln\SoftmaxPolicy_{\LearnedParams}(\action_i \| \context_i)}{\max\{ \prob_i, \MinPS \}}
	+ \frac{\norm{\LearnedParams - \LoggingParams}^2}{\variance_0 \, \MinPS \( \DataSize - 1 \)} ,
\label{eq:mixed_logit_bcrm_convex_objective}
\end{equation}
\begin{equation}
\text{and}\quad
\variance \defeq \min\left\{ \frac{2 \FeatDim}{\FeatMaxNorm^2 \MinPS (\DataSize - 1)} \( \frac{1}{\DataSize} \sum_{i=1}^{\DataSize} \frac{\reward_i}{\max\{ \prob_i, \MinPS \}} \)^{-1} \!\! , \, \variance_0 \, \right\} .
\label{eq:optimal_variance}
\end{equation}
\end{proposition}
\noindent
We defer the proofs of \cref{prop:mixed_logit_bcrm_objective,prop:mixed_logit_bcrm_convex_objective} to \cref{sec:proof_mixed_logit_bcrm_objective}.

Conveniently, \cref{eq:mixed_logit_bcrm_convex_objective} is equivalent to a weighted softmax regression with a modified $L_2$ regularizer. This optimization can be solved using standard methods, with guaranteed convergence to a global optimum. Moreover, by decoupling the optimizations of $\LearnedParams$ and $\variance$ in the upper bound (refer to the proof for details), we can solve for the optimal $\variance$ in closed form.

\begin{remark}
\label{rem:log_transform_connections}
The logarithmic transformation of the target policy in \cref{eq:mixed_logit_bcrm_convex_objective} has interesting connections to other work. Those familiar with reinforcement learning may recognize a similarity to \define{policy gradient} methods. By the policy gradient theorem \citep{sutton:nips00}, the gradient of the expected reward is precisely the expected, reward-weighted gradient of the log-likelihood,\footnote{In reinforcement learning, the expectation would be over trajectories, which we omit for simplicity.}
\begin{equation}
\label{eq:policy_gradient}
\grad \Ep_{\action \by \Policy(\context)}[ \Reward(\context, \action) ] = \Ep_{\action \by \Policy(\context)}[ \Reward(\context, \action) \grad \ln \Policy(\action \| \context) ] .
\end{equation}
In online, on-policy training, the expectation is typically approximated by sampling actions from the policy. In offline, off-policy training, the expectation can be approximated by samples from the logging policy, with importance weight $\Policy(\action \| \context) / \LoggingPolicy(\action \| \context)$ to counteract bias. We then obtain a gradient that looks like the gradient of \cref{eq:mixed_logit_bcrm_convex_objective}, albeit weighted by $\Policy(\action \| \context)$ and without the regularization term.

A similar log-transformed objective was independently derived by \citet{ma:aistats19} as a lower bound to the policy improvement objective (i.e., the gain in reward relative to the logging policy). Whereas our bound retains the IPS scaling, theirs isolates the propensities in an additive term, thereby making them irrelevant to optimization.
\end{remark}

In practice, one usually tunes the amount of regularization to optimize the empirical risk on a held-out validation dataset. By \cref{prop:mixed_logit_bcrm_objective,prop:mixed_logit_bcrm_convex_objective}, this is equivalent to tuning the variance of the prior, $\variance_0$. Though $\LoggingParams$ could in theory be any fixed vector, the case when it is the parameters of the logging policy corresponds to an interesting regularizer. This regularizer instructs the learning algorithm to keep the learned policy close to the logging policy, which effectively reduces the variance of the estimator.

Using \cref{th:mixed_logit_risk_bound}, we can examine how the parameters $\variance_0$ and $\variance$ affect the bias-variance trade-off. Recall from \cref{lem:mixed_logit_mean_bounds} that higher values of $\variance$ increase the bias of the estimator, $\ClipEmpRisk(\LearnedParams, \variance, \Data)$. To reduce this bias, we want $\variance$ to be small; e.g., $\variance = \BigTheta(\DataSize^{-1})$ results in negligible bias. However, if we also have $\variance_0 = \BigTheta(1)$, then $\KLBound(\LoggingParams, \variance_0, \LearnedParams, \variance)$---which can be interpreted as the variance of the estimator---has a term, $\FeatDim \ln\frac{\variance_0}{\variance} = \BigO(\FeatDim \ln \DataSize)$, that depends linearly on the number of features, $\FeatDim$. When $\FeatDim$ is large, this term can dominate the risk bound. The dependence on $\FeatDim$ is eliminated when $\variance_0 = \variance$; but if $\variance_0 = \BigTheta(\DataSize^{-1})$, then $\KLBound(\LoggingParams, \variance_0, \LearnedParams, \variance) = \BigO( \norm{\LearnedParams - \LoggingParams}^2 \DataSize )$, which makes the risk bound vacuous.

\section{When the Logging Policy Is Unknown}
\label{sec:unknown_logging_policy}

In \cref{sec:mixed_logit_model}, we assumed that the logging policy was known and used it to construct a prior, which motivated logging policy regularization. However, there may be settings in which the logging policy is unknown.\footnote{Another motivating scenario is when the logging policy is not in the same class as the new policy; e.g., when the features change.} We can nonetheless construct a prior that approximates the logging policy by learning from its logged actions, which motivates a data-dependent variant of logging policy regularization.

At first, this idea may sound counterintuitive. After all, the prior is supposed to be fixed before drawing the training data. However, the expected value of a function of the data is constant with respect to any realization of the data. Therefore, the expected estimator of the logging policy is independent of the training data, and can serve as a valid prior. This type of prior, known as \define{distribution-dependent}, was introduced by \citet{catoni:ims07} and later developed by others \citep{lever:alt10,hernandez:jmlr12,rivasplata:nips18,dziugaite:nips18} to obtain tight PAC-Bayesian bounds. If the estimator of the logging policy concentrates around its mean, then we can probabilistically bound the distance between the learned logging policy and the distribution-dependent prior. We can then relate the posterior to the learned logging policy via the triangle inequality.

Overloading our previous notation, let $\Loss : \Reals^{\FeatDim} \times \Contexts \times \Actions \to \Reals_+$ denote a loss function that measures the fit of parameters $\params \in \Reals^{\FeatDim}$, given context $\context \in \Contexts$ and action $\action \in \Actions$. We will assume that $\Loss$ is both convex and $\lipschitz$-Lipschitz with respect to $\params$. This assumption is satisfied by, e.g., the negative log-likelihood. For a dataset, $\Data \by (\ContextDistribution \times \LoggingPolicy)^{\DataSize}$, containing logged contexts and actions, let
\begin{equation}
\RERMObjective(\params, \Data) \defeq \frac{1}{\DataSize} \sum_{i=1}^{\DataSize} \Loss(\params, \context_i, \action_i) + \regparam \norm{\params}^2
\label{eq:rerm_objective}
\end{equation}
denote a regularized objective; let
\begin{equation}
\RERM \defeq \argmin_{\params \in \Reals^{\FeatDim}} \, \RERMObjective(\params, \Data)
\label{eq:rerm}
\end{equation}
denote its minimizer; and let
\begin{equation}
\MeanRERM \defeq \Ep_{\Data \by (\ContextDistribution \times \LoggingPolicy)^{\DataSize}}[ \RERM[\Data] ]
\label{eq:expected_rerm}
\end{equation}
denote the expected minimizer. Since $\MeanRERM$ is a constant, it is independent of any realization of $\Data$. We can therefore construct a Gaussian prior around $\MeanRERM$, which makes the KL divergence proportional to $\norm{\LearnedParams - \MeanRERM}^2$.

Importantly, $\RERMObjective$ is strongly convex. An implication of this property is that its minimizer exhibits \define{uniform algorithmic stability}; meaning, it is robust to perturbations of the training data. Using stability, one can show that the random variable $\RERM$ concentrates around its mean, $\MeanRERM$ \citep{liu:icml17}. Thus, with high probability, $\norm{\RERM - \MeanRERM}$ is small---which, via the triangle inequality, implies that $\norm{\LearnedParams - \MeanRERM}$ is approximately $\norm{\LearnedParams - \RERM}$.

We use this reasoning to prove the following (deferred to \cref{sec:proof_data_dep_reg_risk_bound}).
\begin{theorem}
\label{th:data_dep_reg_risk_bound}
Let $\Hypotheses$ denote the hypothesis space defined in \cref{eq:mixed_logit_class,eq:sampler}, and let $\Policy_{\Posterior}$ denote the mixed logit policy defined in \cref{eq:mixed_logit_policy}. Let $\RERM$ denote the minimizer defined in \cref{eq:rerm}, for a convex, $\lipschitz$-Lipschitz loss function. For any $\DataSize \geq 1$, $\delta \in (0,1)$, $\MinPS \in (0,1)$ and $\variance_0 \in (0,\infty)$, with probability at least $1 - \delta$ over draws of $\Data \by (\ContextDistribution \times \LoggingPolicy)^{\DataSize}$, the following holds simultaneously for all $\LearnedParams \in \Reals^{\FeatDim}$ and $\variance \in (0,\variance_0]$:
\begin{align}
\Risk(\Policy_{\Posterior})
	\leq \ClipEmpRisk(\LearnedParams, \variance, \Data)
	&+ \sqrt{ \frac{ \big( \ClipEmpRisk(\LearnedParams, \variance, \Data) - 1 + \frac{1}{\MinPS} \big) \big( \hat{\KLBound}(\RERM, \variance_0, \LearnedParams, \variance) + 2 \ln\frac{2\DataSize}{\delta} \big) }{ \MinPS \( \DataSize - 1 \) } } \\
	&+ \frac{ \hat{\KLBound}(\RERM, \variance_0, \LearnedParams, \variance) + 2 \ln\frac{2\DataSize}{\delta} }{ \MinPS \( \DataSize - 1 \) } ,
\label{eq:data_dep_reg_risk_bound}
\end{align}
\begin{equation}
\text{where}\quad
\hat{\KLBound}(\RERM, \variance_0, \LearnedParams, \variance)
	\defeq \frac{ \Big( \norm{\LearnedParams - \RERM} + \frac{\lipschitz}{\regparam} \sqrt{ \frac{2 \ln\frac{4}{\delta}}{\DataSize} } \Big)^2 }{\variance_0}
		+ \FeatDim \ln\frac{\variance_0}{\variance} .
\label{eq:data_dep_reg_kl_bound}
\end{equation}
\end{theorem}

It is straightforward to show that \cref{prop:mixed_logit_bcrm_objective,prop:mixed_logit_bcrm_convex_objective} hold for \cref{th:data_dep_reg_risk_bound} with $\LoggingParams \defeq \RERM$. Thus, \cref{th:data_dep_reg_risk_bound} motivates the following two-step learning procedure for Bayesian CRM:
\begin{enumerate}
\itemsep 2pt
\item Using logged data, $\Data$, but ignoring rewards, solve \cref{eq:rerm} to estimate softmax parameters, $\RERM$, that approximate the logging policy.
\item Using $\Data$ again, including the rewards, solve \cref{eq:mixed_logit_bcrm_objective} or \ref{eq:mixed_logit_bcrm_convex_objective}, with $\LoggingParams \defeq \RERM$, to train a new mixed logit policy.
\end{enumerate}

\begin{remark}
\label{rem:missing_propensities}
Throughout, we have assumed that the logged data includes the propensities, which enable IPS weighting. Given that we can learn to approximate the logging policy, it seems natural to use the learned propensities in the absence of the true propensities. In practice, this approximation may work, though we cannot provide any formal guarantees for it without making assumptions about the true logging policy. We leave this as a task for future work.
\end{remark}

\section{Experiments}
\label{sec:experiments}

Our Bayesian analysis of CRM suggests a new regularization technique, logging policy regularization (LPR). Using the logging policy to construct a prior, we regularize by the squared distance between the (learned) logging policy's softmax parameters, $\LoggingParams$, and the posterior mean, $\LearnedParams$, over softmax parameters. In this section, we empirically verify the following claims:
\begin{enumerate}
\itemsep 2pt
\item LPR outperforms standard $L_2$ regularization whenever the logging policy outperforms a uniform action distribution.
\item LPR is competitive with variance regularization (i.e., POEM \citep{swaminathan:jmlr15}), and is also faster to optimize.
\item When the logging policy is unknown, we can estimate it from logged data, then use the estimator in LPR with little deterioration in performance.
\end{enumerate}

We will use the class of mixed logit models from \cref{sec:mixed_logit_model}. For simplicity, we choose to only optimize the posterior mean, $\LearnedParams$, assuming that the posterior variance, $\variance$, is fixed to some small value, e.g., $\DataSize^{-1}$. This is inconsequential, since we will approximate the posterior action probabilities, $\Policy_{\Posterior}(\action \| \context)$, with a softmax of the mean parameters, $\SoftmaxPolicy_{\LearnedParams}(\action \| \context)$. By \cref{lem:mixed_logit_mean_bounds}, with small $\variance$, this is a reasonable approximation. In a small departure from our analysis, we add an unregularized bias term for each action.

We evaluate two methods based on LPR. The first method, inspired by \cref{prop:mixed_logit_bcrm_objective}, combines LPR with the truncated IPS estimator:
\begin{equation}
\argmin_{\LearnedParams \in \Reals^{\FeatDim}} \,
	\frac{1}{\DataSize} \sum_{i=1}^{\DataSize} - \frac{\reward_i \, \SoftmaxPolicy_{\LearnedParams}(\action_i \| \context_i)}{\max\{ \prob_i, \MinPS \}}
	+ \regparam \norm{\LearnedParams - \LoggingParams}^2 ,
\label{eq:ips_lpr}
\end{equation}
where $\MinPS \in (0,1)$ and $\regparam \geq 0$ are free parameters. (Note that we have omitted the constant one from the empirical risk, since it is irrelevant to the optimization.) We call this method IPS-LPR. The second method, inspired by \cref{prop:mixed_logit_bcrm_objective}, is a convex upper bound:
\begin{equation}
\argmin_{\LearnedParams \in \Reals^{\FeatDim}} \,
	\frac{1}{\DataSize} \sum_{i=1}^{\DataSize} - \frac{\reward_i \ln\SoftmaxPolicy_{\LearnedParams}(\action_i \| \context_i)}{\max\{ \prob_i, \MinPS \}}
	+ \regparam \norm{\LearnedParams - \LoggingParams}^2 .
\label{eq:wnll_lpr}
\end{equation}
Since the first term is essentially a weighted negative log-likelihood, we call this method WNLL-LPR. 

We compare the above methods to several baselines. The first baseline is IPS with standard $L_2$ regularization,
\begin{equation}
\argmin_{\LearnedParams \in \Reals^{\FeatDim}} \,
	\frac{1}{\DataSize} \sum_{i=1}^{\DataSize} - \frac{\reward_i \, \SoftmaxPolicy_{\LearnedParams}(\action_i \| \context_i)}{\max\{ \prob_i, \MinPS \}}
	+ \regparam \norm{\LearnedParams}^2 ,
\label{eq:ips_l2}
\end{equation}
which we refer to as IPS-L2. The second baseline is POEM \citep{swaminathan:jmlr15}, which solves a variance regularized objective,
\begin{equation}
\argmin_{\LearnedParams \in \Reals^{\FeatDim}} \,
	\frac{1}{\DataSize} \sum_{i=1}^{\DataSize} - \reward_i \min\left\{ \frac{\SoftmaxPolicy_{\LearnedParams}(\action_i \| \context_i)}{\prob_i}, \MinPS^{-1}  \right\}
	+ \regparam \sqrt{\frac{\SampleVar_{\MinPS}(\SoftmaxPolicy_{\LearnedParams}, \Data)}{\DataSize}} ,
\label{eq:poem}
\end{equation}
using a majorization-minimization algorithm. We also test a variant of POEM that adds $L_2$ regularization, which we refer to as POEM-L2.

All methods require some form of IPS truncation. For IPS-L2, IPS-LPR and WNLL-LPR, we use $\max\{ \prob_i, \MinPS \}^{-1}$; for POEM and POEM-L2, we use $\min\{ \Policy(\action_i \| \context_i) / \prob_i, \MinPS^{-1} \}$, per \citeauthor{swaminathan:jmlr15}'s original formulation. In all experiments, we set $\MinPS \defeq 0.01$, which concurs with \citepos{ionides:jcgs08} recommendation of $\BigO(\DataSize^{-1})$.

Since all methods support stochastic first-order optimization, we use AdaGrad \citep{duchi:jmlr11} with mini-batches of 100 examples. We set the learning rate to 0.1 and the smoothing parameter to one, which we find necessary to avoid numerical instability due to small gradients in early rounds of training. Unless otherwise stated, we run training for 500 epochs, with random shuffling of the training data at each epoch. All model parameters are initialized to zero, and all runs of training are seeded such that every method receives the same sequence of training examples.

We report results on two benchmark image classification datasets: Fashion-MNIST \citep{xiao:corr17} and CIFAR-100 \citep{krizhevsky:tech09}. Fashion-MNIST consists of 70,000 (60,000 training; 10,000 testing) grayscale images from 10 categories of apparel and accessories. We extract features from each image by normalizing pixel intensities to $[0,1]$ and flattening the $(28 \times 28)$-pixel grid to a 784-dimensional vector. CIFAR-100 consists of 60,000 (50,000 training; 10,000 testing) color images from 100 general object categories. As this data is typically modeled with deep convolutional neural networks, we use transfer learning to extract features expressive enough to yield decent performance with the class of log-linear models described in \cref{sec:mixed_logit_model}.\footnote{In practice, there is no reason why one could not simply learn the representation and policy jointly---e.g., using a convolutional neural network with softmax output---but we chose to keep our experiments as close as possible to what is supported by our theoretical analysis.} Specifically, we use the last hidden layer of a pre-trained ResNet-50 network \citep{he:cvpr16}, which was trained on ImageNet \citep{deng:cvpr09}, to output 2048-dimensional features for CIFAR-100.

Following prior work \citep{beygelzimer:kdd09,swaminathan:jmlr15,joachims:iclr18}, we use a standard supervised-to-bandit conversion to simulate logged bandit feedback. We start by randomly sampling 1,000 training examples (without replacement) to train a softmax logging policy using supervised learning. We then use the logging policy to sample a label (i.e., action) for each remaining training example. The reward is one if the sampled label matches the true label, and zero otherwise. We repeat this procedure 10 times, using 10 random splits of the training data, thereby generating 10 datasets of logged contexts, actions, propensities and rewards.

We compare methods along two metrics. Our primary metric is the expected reward under the stochastic policy, $\Ep_{\action \by \Policy(\context)}[ \Reward(\context, \action) ]$, averaged over the testing data. Our secondary metric---which is not directly supported by our analysis, but is nonetheless of interest---is the reward of the deterministic \define{argmax policy}, $\Reward(\context, \argmax_{\action\in\Actions} \Policy(\action \| \context))$. Since the reward is one for the true label and zero otherwise, the first metric is simply the policy's probability of sampling the correct label, and the second metric is the accuracy of the argmax policy.

\subsection{Logging Policy as Prior}
\label{sec:logging_policy_prior}

Our first experiment investigates our claim that the logging policy is a better prior than a standard normal distribution, thus motivating LPR over $L_2$ regularization. For each simulated log dataset, we train new policies using IPS-L2 and IPS-LPR, with regularization parameter values $\regparam = 10^{-6}, 10^{-5}, \dots, 1$. \cref{fig:l2_vs_lpr-reg_param_sweep} plots the expected test reward as a function of $\regparam$. The dotted, black line indicates the performance of the logging policy. We find that IPS-LPR outperforms IPS-L2 for each value of $\regparam$; meaning, for any amount of regularization, IPS-LPR is always better. Further, while the performance of IPS-L2 degrades to that of a uniform action distribution as we over-regularize, the performance of IPS-LPR converges to that of the logging policy. This illustrates the natural intuition that a policy that does something smarter than random guessing is an informative prior.
 
An implication of this statement is that, as the logging policy's action distribution becomes more uniform, its efficacy as a prior should diminish. To verify this, we construct a sequence of logging policies that interpolate between the above logging policy and the uniform distribution. We do so by multiplying the weights by an inverse-temperature parameter, $\invtemp = 0, 0.2, \dots, 1$. We then generate log datasets for each logging policy, and train new policies using IPS-L2 and IPS-LPR, with $\regparam \defeq 0.001$. \cref{fig:l2_vs_lpr-invtemp_sweep} plots the resulting test reward as a function of $\invtemp$. As expected, the performance of IPS-LPR gradually converges to that of IPS-L2 as the logging policy converges to the uniform distribution.

One could also ask what happens when when the logging policy is \emph{worse} than a uniform distribution. Indeed, though not shown here, we find that IPS-LPR performs worse than IPS-L2 in that scenario. However, one could reasonably argue that such a scenario is unlikely to occur in practice, since there is no point to deploying a logging policy that performs worse than a uniform distribution. If the uniform distribution achieves higher reward, then it makes more sense to deploy it and thereby collect unbiased data. In the setting where we have data to train the logging policy, it is straightforward to estimate its expected reward and compare it to that of the uniform distribution.

\begin{figure}[t]
\centering
\subfloat[][Varying the amount of regularization.]{
	\includegraphics[width=0.235\textwidth]{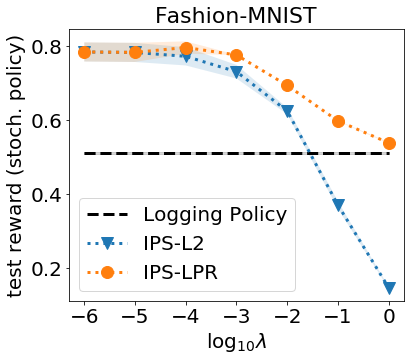}
	\includegraphics[width=0.235\textwidth]{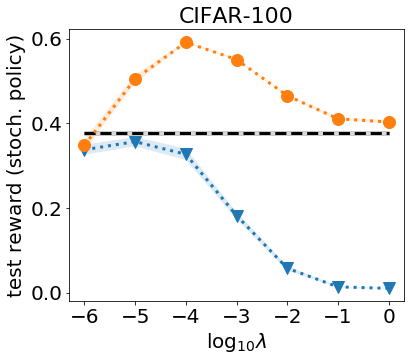}
	\label{fig:l2_vs_lpr-reg_param_sweep}
	}
\subfloat[][Varying the quality of the logging policy.]{
	\includegraphics[width=0.25\textwidth]{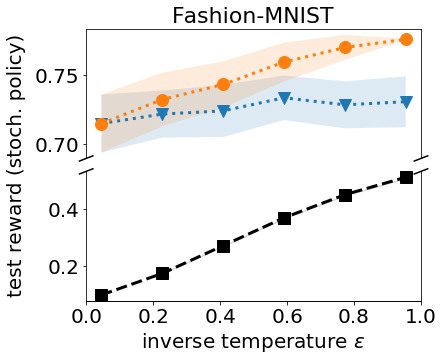}
	\includegraphics[width=0.235\textwidth]{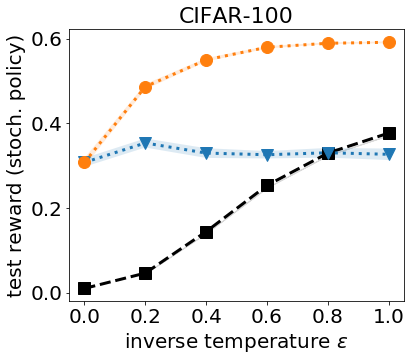}
	\label{fig:l2_vs_lpr-invtemp_sweep}
	}
\caption{$L_2$ regularization vs.\ logging policy regularization (LPR). Each line is the average of 10 trials, with shading to indicate the 95\% confidence interval. \cref{fig:l2_vs_lpr-reg_param_sweep} plots the expected test reward as a function of the regularization parameter, $\regparam$. \cref{fig:l2_vs_lpr-invtemp_sweep} analyzes a spectrum of logging policies from the uniform action distribution ($\invtemp = 0$) to the trained distribution ($\invtemp = 1$).}
\label{fig:l2_vs_lpr}
\end{figure}

\subsection{Comparison to POEM}
\label{sec:comparison_to_poem}

As discussed earlier, LPR relates to variance regularization in that one way to minimize variance is to keep the new policy close to the logging policy. We are therefore prompted to investigate how LPR compares to variance regularization (i.e., POEM) in practice. In this experiment, our goal is to achieve the highest expected reward for each method on each log dataset, without looking at the testing data. Accordingly, we tune the regularization parameter, $\regparam$, using 5-fold cross-validation on each log dataset, with truncated IPS estimation of expected reward on the holdout set. For simplicity, we use grid search over powers of ten; $\regparam = 10^{-8}, \dots, 10^{-3}$ for LPR and $\regparam = 10^{-3}, \dots, 10^{2}$ for variance regularization. For POEM-L2, we tune the $L_2$ regularization parameter (in the same range as LPR) by fixing the variance regularization parameter to its optimal value. During parameter tuning, we limit training to 100 epochs. Once the parameter values have been selected, we train a new policy on the entire log dataset for 500 (Fashion-MNIST) or 1000 (CIFAR-100) epochs and evaluate it on the testing data.

\cref{tab:all_results} reports the results of this experiment. For completeness, we include results for all proposed methods and baselines, including the logging policy. On Fashion-MNIST, the variance regularization baselines (POEM and POEM-L2) achieve the highest expected reward, but the LPR methods (IPS-LPR and WNLL-LPR) are competitive. Indeed, the differences between these methods are not statistically significant according to a paired $t$-test with significance threshold 0.05. Meanwhile, all four significantly outperform IPS-L2 and the logging policy. Interestingly, WNLL-LPR performs best in terms of the argmax policy, perhaps owing to the fact that it is optimizing what is essentially a classification loss. Indeed, in classification problems with bandit feedback and binary rewards, the first term in \cref{eq:wnll_lpr} is an unbiased estimator of the expected negative log-likelihood, which is a surrogate for the expected misclassification rate of the argmax policy.

The CIFAR-100 data presents a more challenging learning problem than Fashion-MNIST, since it has a much larger action set, and several times as many features. It is perhaps due to these difficulties that the baselines are unable to match the performance of the logging policy---which, despite being trained on far less data, is trained with full supervision. Meanwhile, both LPR methods outperform the logging policy by wide margins. We believe this is due to the fact that LPR is designed with incremental training in mind. The new policy is encouraged to stay close to the logging policy not just to hedge against overfitting, but also because the logging policy is assumed to be a good starting point.

\begin{table}[t]
\caption{Test set rewards for Fashion-MNIST and CIFAR-100, averaged over 10 trials, with 5-fold cross-validation of regularization parameters at each trial.}
\centering
\begin{tabular}{lcccc}
\toprule
  			& \multicolumn{2}{c}{\textbf{Fashion-MNIST}}	& \multicolumn{2}{c}{\textbf{CIFAR-100}} \\
\textbf{Method} &  \textbf{stoch.} &  \textbf{argmax} 		      	&  \textbf{stoch.} &  \textbf{argmax} \\
\midrule
Logging Policy	&      0.5123 &    0.7099					&	0.3770 &    0.4797 \\
\midrule
IPS-L2         	&      0.7778 &    0.7890					&      0.3475 &    0.3624 \\
POEM		&      0.8060 &    0.8124					& 	0.3338 &    0.3392 \\
POEM-L2		&	0.8050 &    0.8126					&	0.3486 &    0.3641 \\
\midrule
IPS-LPR        	&	0.7955 &    0.8154					&	0.5553 &    0.6134 \\
WNLL-LPR       &      0.7978 &    0.8305					&	0.6143 &    0.6272 \\
\midrule
IPS-LLPR		&      0.7950 &    0.8153					& 	0.5455  & 	  0.6077 \\
WNLL-LLPR     &      0.7978 &    0.8305					& 	0.6143  & 	  0.6272 \\
\bottomrule
\end{tabular}
\label{tab:all_results}
\end{table}

It is worth comparing the running times of POEM and LPR. Recall that POEM is a majorization-minimization algorithm designed to enable stochastic optimization of a variance-regularized objective. At each epoch of training, POEM constructs an upper bound to the objective by processing all examples in the training data. This additional computation effectively doubles POEM's time complexity relative to the LPR methods, which only require one pass over the data per epoch. In the Fashion-MNIST experiments, we find that POEM is on average $25\%$ slower than IPS-LPR.

\subsection{Learning the Logging Policy}
\label{sec:learn_logging_policy}

Per \cref{sec:unknown_logging_policy}, when the logging policy is unknown, we can estimate its softmax parameters, $\LoggingParams$, then use the estimate, $\RERM$, in LPR. We now verify this claim empirically on Fashion-MNIST. Using the log datasets from the previous sections, we learn the logging policy with the regularized negative log-likelihood:
\begin{equation}
\RERM \defeq \argmin_{\params \in \Reals^{\FeatDim}} \, \frac{1}{\DataSize} \sum_{i=1}^{\DataSize} -\ln \SoftmaxPolicy_{\params}(\action_i \| \context_i) + \regparam \norm{\params}^2 .
\label{eq:learn_log_policy_objective}
\end{equation}
We optimize this objective using 100 epochs of AdaGrad, with the same settings as the other experiments. We set the regularization parameter aggressively high, $\regparam \defeq 0.01$, to ensure that the learned distribution does not become too peaked. Given $\RERM$ for each log dataset, we then train new policies using IPS-LPR and WNLL-LPR, with the same $\regparam$ values tuned in \cref{sec:comparison_to_poem}. The results of this experiment are given in the bottom section of \cref{tab:all_results}, as methods IPS-LLPR and WNLL-LLPR (for \emph{learned} LPR). The rewards are nearly identical to those when the logging policy is known, thus demonstrating that LPR does not require the actual logging policy in order to be effective.

\section{Conclusion}
\label{sec:conclusion}

We have presented a PAC-Bayesian analysis of counterfactual risk minimization, for learning Bayesian policies from logged bandit feedback. Like \citepos{swaminathan:jmlr15} risk bound (\cref{eq:crm_bound}), ours achieves a ``fast" learning rate under certain conditions---though theirs suggests variance regularization, while ours suggests regularizing by the posterior's divergence from the prior. We applied our risk bound to a class of mixed logit policies, from which we derived two Bayesian CRM objectives based on logging policy regularization. We also derived a two-step learning procedure for estimating the regularizer when the logging policy is unknown. Our empirical study indicated that logging policy regularization can achieve significant improvements over existing methods.

\section*{Acknowledgements}
\label{sec:acks}

We thank Thorsten Joachims for thoughtful discussions and helpful feedback.


\appendix

\section{Deferred Proofs}
\label{sec:deferred_proofs}

This appendix contains all deferred proofs.

\subsection{Proof of \cref{th:pac_bayes_bound_allclipping}}
\label{sec:proof_pac_bayes_bound_allclipping}

We construct an infinite sequence of $\MinPS$ values, $(\MinPS_i \defeq 2^{-i})_{i=1}^{\infty}$, and $\delta$ values, $(\delta_i \defeq \delta \MinPS_i)_{i=1}^{\infty}$. For any $\MinPS_i$, \cref{eq:pac_bayes_bound} holds with probability at least $1 - \delta_i$. Thus, with probability at least $1 - \sum_{i=0}^\infty \delta_i = 1 - \delta$, \cref{eq:pac_bayes_bound} holds for all $\MinPS_i$ simultaneously.

For a given $\MinPS$---which may depend on the data---we select $i^\star \defeq \ceil{ \frac{\ln \MinPS^{-1}}{\ln 2} }$. (Since $\MinPS \in (0,1)$, the ceiling function ensures that $i^\star \geq 1$.) Then, we have that $\MinPS / 2 \leq \MinPS_{i^\star} \leq \MinPS$; and, since $\max\{ \prob, \MinPS_{i^\star} \} \leq \max\{ \prob, \MinPS \}$, we have that $\ClipEmpRisk[\MinPS_{i^\star}](\Policy, \Data) \leq \ClipEmpRisk(\Policy, \Data)$. Further, $\delta_{i^\star} \geq \delta \MinPS / 2$. Thus, with probability at least $1 - \delta$,
\begin{align}
\Risk(\Policy)
	\leq \ClipEmpRisk[\MinPS_{i^\star}](\Policy, \Data)
		&+ \sqrt{ \frac{ 2 \big( \ClipEmpRisk[\MinPS_{i^\star}](\Policy, \Data) - 1 + \frac{1}{\MinPS_{i^\star}} \big) \big( \KLDiv{\Posterior}{\Prior} + \ln\frac{\DataSize}{\delta_{i^\star}} \big) }{ \MinPS_{i^\star} \( \DataSize - 1 \) } } \\
		&+ \frac{ 2 \( \KLDiv{\Posterior}{\Prior} + \ln\frac{\DataSize}{\delta_{i^\star}} \) }{ \MinPS_{i^\star} \( \DataSize - 1 \) } \\
	\leq \ClipEmpRisk(\Policy, \Data)
		&+ \sqrt{ \frac{ 4 \big( \ClipEmpRisk(\Policy, \Data) - 1 + \frac{2}{\MinPS} \big) \( \KLDiv{\Posterior}{\Prior} + \ln\frac{2 \DataSize}{\delta \MinPS} \) }{ \MinPS \( \DataSize - 1 \) } } \\
		&+ \frac{ 4 \( \KLDiv{\Posterior}{\Prior} + \ln\frac{2 \DataSize}{\delta \MinPS} \) }{ \MinPS \( \DataSize - 1 \) } ,
\label{eq:pac_bayes_bound_i_star}
\end{align}
which completes the proof.

\subsection{Proof of \cref{lem:mixed_logit_mean_bounds}}
\label{sec:proof_mixed_logit_mean_bounds}

We begin with the lower bound. First, let 
\begin{equation}
\Partition(\params) \defeq \sum_{\action' \in \Actions} \exp( \params \cdot \Feats(\context, \action') )
\label{eq:partition}
\end{equation}
denote a normalizing constant, sometimes referred to as the \define{partition function}. (Since $\context$ is given, our notation ignores the fact that $\Partition$ is a function of $\context$.) Using $\Partition$ in the definition of $\SoftmaxPolicy$, and applying Jensen's inequality, we have that
\begin{align}
\Policy_{\Posterior}(\action \| \context)
	&= \Ep_{\params \by \Norm(\LearnedParams, \variance \mat{I})} \left[ \SoftmaxPolicy_{\params}(\action \| \context) \right] \\
	&= \Ep_{\params \by \Norm(\LearnedParams, \variance \mat{I})} \left[ \exp\( \params \cdot \Feats(\context, \action) - \ln\Partition(\params) \) \right] \\
	&\geq \exp\( \Ep_{\params \by \Norm(\LearnedParams, \variance \mat{I})} \left[ \params \cdot \Feats(\context, \action) - \ln\Partition(\params) \right] \) .
\label{eq:post_action_prob_lb_1}
\end{align}
We then express the random parameters, $\params \by \Norm(\LearnedParams, \variance \mat{I})$, as the sum of the mean parameters, $\LearnedParams$, and a zero-mean Gaussian vector, $\gaussians \by \Norm(0, \variance \mat{I})$, which yields
\begin{align}
\Ep_{\params \by \Norm(\LearnedParams, \variance \mat{I})} \left[ \params \cdot \Feats(\context, \action) - \ln\Partition(\params) \right]
	&= \Ep_{\gaussians \by \Norm(0, \variance \mat{I})} \left[ (\LearnedParams + \gaussians) \cdot \Feats(\context, \action) - \ln\Partition(\LearnedParams + \gaussians) \right] \\
	&= \LearnedParams \cdot \Feats(\context, \action) \, - \Ep_{\gaussians \by \Norm(0, \variance \mat{I})} \left[ \ln\Partition(\LearnedParams + \gaussians) \right] \\
	&= \LearnedParams \cdot \Feats(\context, \action) - \ln\Partition(\LearnedParams) \\
	&\quad - \Ep_{\gaussians \by \Norm(0, \variance \mat{I})} \left[ \ln\( \frac{\Partition(\LearnedParams + \gaussians)}{\Partition(\LearnedParams)} \) \right] .
\label{eq:post_action_prob_lb_2}
\end{align}
The second line follows from the fact that the expected dot product of any vector with a zero-mean Gaussian vector is zero. Applying Jensen's inequality again to the last term, we have
\begin{equation}
- \Ep_{\gaussians \by \Norm(0, \variance \mat{I})} \left[ \ln\( \frac{\Partition(\LearnedParams + \gaussians)}{\Partition(\LearnedParams)} \) \right]
	\geq - \ln \Ep_{\gaussians \by \Norm(0, \variance \mat{I})} \left[ \frac{\Partition(\LearnedParams + \gaussians)}{\Partition(\LearnedParams)} \right] .
\label{eq:post_action_prob_lb_3}
\end{equation}
Observe that
\begin{equation}
\frac{\Partition(\LearnedParams + \gaussians)}{\Partition(\LearnedParams)}
	= \sum_{\action' \in \Actions} \frac{\exp(\LearnedParams \cdot \Feats(\context, \action') )}{\Partition(\LearnedParams)} \, \exp( \gaussians \cdot \Feats(\context, \action') )
	= \Ep_{\action' \by \SoftmaxPolicy_{\LearnedParams}(\context)} \left[ \exp( \gaussians \cdot \Feats(\context, \action') ) \right] .
\label{eq:post_action_prob_lb_4}
\end{equation}
Thus, via linearity of expectation,
\begin{equation}
\Ep_{\gaussians \by \Norm(0, \variance \mat{I})} \left[ \frac{\Partition(\LearnedParams + \gaussians)}{\Partition(\LearnedParams)} \right]
	= \Ep_{\action' \by \SoftmaxPolicy_{\LearnedParams}(\context)} \Ep_{\gaussians \by \Norm(0, \variance \mat{I})} \left[ \exp( \gaussians \cdot \Feats(\context, \action') ) \right] .
\label{eq:post_action_prob_lb_5}
\end{equation}
The right-hand inner expectation is simply the moment-generating function of a multivariate Gaussian. Combining its closed-form expression,
\begin{equation}
\Ep_{\gaussians \by \Norm(0, \variance \mat{I})} \left[ \exp( \gaussians \cdot \Feats(\context, \action') ) \right]
	= \exp\( \frac{\variance}{2} \norm{\Feats(\context, \action')}^2 \) ,
\label{eq:post_action_prob_lb_6}
\end{equation}
with \cref{eq:post_action_prob_lb_5}, we have
\begin{align}
- \ln \Ep_{\gaussians \by \Norm(0, \variance \mat{I})} \left[ \frac{\Partition(\LearnedParams + \gaussians)}{\Partition(\LearnedParams)} \right]
	&= - \ln \Ep_{\action' \by \SoftmaxPolicy_{\LearnedParams}(\context)} \left[ \exp\( \frac{\variance}{2} \norm{\Feats(\context, \action')}^2 \) \right] \\
	&\geq - \ln \Ep_{\action' \by \SoftmaxPolicy_{\LearnedParams}(\context)} \left[ \exp\( \frac{\variance \FeatMaxNorm^2}{2} \) \right] \\
	&= - \frac{\variance \FeatMaxNorm^2}{2} .
\label{eq:post_action_prob_lb_7}
\end{align}
The inequality follows from the assumption that $\norm{\Feats(\context, \action')} \leq \FeatMaxNorm$. Finally, combining \cref{eq:post_action_prob_lb_1,eq:post_action_prob_lb_2,eq:post_action_prob_lb_3,eq:post_action_prob_lb_7}, we have
\begin{align}
\Policy_{\Posterior}(\action \| \context)
	&\geq \exp\( \Ep_{\params \by \Norm(\LearnedParams, \variance \mat{I})} \left[ \params \cdot \Feats(\context, \action) - \ln\Partition(\params) \right] \) \\
	&= \exp\( \LearnedParams \cdot \Feats(\context, \action) - \ln\Partition(\LearnedParams) - \Ep_{\gaussians \by \Norm(0, \variance \mat{I})} \left[ \ln\( \frac{\Partition(\LearnedParams + \gaussians)}{\Partition(\LearnedParams)} \) \right] \) \\
	&\geq \exp\( \LearnedParams \cdot \Feats(\context, \action) - \ln\Partition(\LearnedParams) - \ln \Ep_{\gaussians \by \Norm(0, \variance \mat{I})} \left[ \frac{\Partition(\LearnedParams + \gaussians)}{\Partition(\LearnedParams)} \right] \) \\
	&\geq \exp\( \LearnedParams \cdot \Feats(\context, \action) - \ln\Partition(\LearnedParams) - \frac{\variance \FeatMaxNorm^2}{2} \) \\
	&= \SoftmaxPolicy_{\LearnedParams}(\action \| \context) \exp\( -\frac{\variance \FeatMaxNorm^2}{2} \) .
\end{align}

To prove the upper bound, first observe that
\begin{align}
\SoftmaxPolicy_{\params}(\action \| \context)
	&= \exp\( \LearnedParams \cdot \Feats(\context, \action) - \ln\Partition(\LearnedParams) + \gaussians \cdot \Feats(\context, \action) - \ln\( \frac{\Partition(\LearnedParams + \gaussians)}{\Partition(\LearnedParams)} \) \) \\
	&= \SoftmaxPolicy_{\LearnedParams}(\action \| \context) \exp\( \gaussians \cdot \Feats(\context, \action) - \ln\( \frac{\Partition(\LearnedParams + \gaussians)}{\Partition(\LearnedParams)} \) \) \\
	&= \SoftmaxPolicy_{\LearnedParams}(\action \| \context) \exp\( \gaussians \cdot \Feats(\context, \action) - \ln\Ep_{\action' \by \SoftmaxPolicy_{\LearnedParams}(\context)} \left[ \exp( \gaussians \cdot \Feats(\context, \action') ) \right] \) \\
	&\leq \SoftmaxPolicy_{\LearnedParams}(\action \| \context) \exp\( \gaussians \cdot \Feats(\context, \action) - \Ep_{\action' \by \SoftmaxPolicy_{\LearnedParams}(\context)} \left[ \gaussians \cdot \Feats(\context, \action') \right] \) \\
	&\leq \SoftmaxPolicy_{\LearnedParams}(\action \| \context) \Ep_{\action' \by \SoftmaxPolicy_{\LearnedParams}(\context)} \left[ \exp\( \gaussians \cdot (\Feats(\context, \action) - \Feats(\context, \action')) \) \right] .
\end{align}
The inequalities follow from Jensen's inequality. We then have that
\begin{align}
\Policy_{\Posterior}(\action \| \context)
	&= \Ep_{\params \by \Norm(\LearnedParams, \variance \mat{I})} \left[ \SoftmaxPolicy_{\params}(\action \| \context) \right] \\
	&\leq \SoftmaxPolicy_{\LearnedParams}(\action \| \context) \Ep_{\action' \by \SoftmaxPolicy_{\LearnedParams}(\context)} \Ep_{\gaussians \by \Norm(0, \variance \mat{I})} \left[ \exp\( \gaussians \cdot (\Feats(\context, \action) - \Feats(\context, \action')) \) \right] .
\end{align}
The right-hand inner expectation is the moment-generating function of a multivariate Gaussian:
\begin{align}
\Ep_{\gaussians \by \Norm(0, \variance \mat{I})} \left[ \exp\( \gaussians \cdot (\Feats(\context, \action) - \Feats(\context, \action')) \) \right]
	&= \exp\( \frac{\variance}{2} \norm{\Feats(\context, \action) - \Feats(\context, \action')}^2 \) \\
	&\leq \exp\( \frac{\variance}{2} (\norm{\Feats(\context, \action)} + \norm{\Feats(\context, \action')})^2 \) \\
	&\leq \exp\( \frac{\variance}{2} (\FeatMaxNorm + \FeatMaxNorm)^2 \) \\
	&= \exp(2 \variance \FeatMaxNorm^2) .
\end{align}
The first inequality follows from the triangle inequality. Therefore,
\begin{equation}
\Policy_{\Posterior}(\action \| \context)
	\leq \SoftmaxPolicy_{\LearnedParams}(\action \| \context) \Ep_{\action' \by \SoftmaxPolicy_{\LearnedParams}(\context)} [ \exp(2 \variance \FeatMaxNorm^2) ]
	= \SoftmaxPolicy_{\LearnedParams}(\action \| \context) \exp(2 \variance \FeatMaxNorm^2) ,
\end{equation}
which completes the proof.

\subsection{Proofs of \cref{prop:mixed_logit_bcrm_objective,prop:mixed_logit_bcrm_convex_objective}}
\label{sec:proof_mixed_logit_bcrm_objective}

We start by proving \cref{prop:mixed_logit_bcrm_objective}. To simplify \cref{eq:mixed_logit_risk_bound}, we let
\begin{equation}
\alpha \defeq \ClipEmpRisk(\LearnedParams, \variance, \Data) - 1 + \frac{1}{\MinPS}
\eqand
\beta \defeq \frac{ \KLBound(\LoggingParams, \variance_0, \LearnedParams, \variance) + 2 \ln\frac{\DataSize}{\delta} }{ \MinPS \( \DataSize - 1 \) }.
\end{equation}
Noting that $\ClipEmpRisk(\LearnedParams, \variance, \Data) \leq \alpha$ (since $\MinPS^{-1} - 1 \geq 0$), we can upper-bound \cref{eq:mixed_logit_risk_bound} as
\begin{equation}
\Risk(\Policy_{\Posterior})
	\leq \alpha + \sqrt{\alpha \beta} + \beta .
\label{eq:geometric_mean}
\end{equation}
The middle term is the geometric mean of $\alpha$ and $\beta$, which is at most the arithmetic mean:
\begin{equation}
\alpha + \sqrt{\alpha \beta} + \beta
	\leq \alpha + \frac{\alpha + \beta}{2} + \beta
	= \frac{3(\alpha + \beta)}{2} .
\label{eq:geometric_mean_to_arithmetic_mean}
\end{equation}
We therefore obtain an upper bound on \cref{eq:mixed_logit_risk_bound} that omits the middle term, which can be tricky to optimize due to the interaction between $\alpha$ and $\beta$. If we optimize this upper bound,
\begin{align}
\argmin_{\substack{\LearnedParams \in \Reals^{\FeatDim} \\ \variance \in (0, \variance_0]}} \, \frac{3(\alpha + \beta)}{2}
	&= \argmin_{\substack{\LearnedParams \in \Reals^{\FeatDim} \\ \variance \in (0, \variance_0]}} \, \alpha + \beta \\
	&= \argmin_{\substack{\LearnedParams \in \Reals^{\FeatDim} \\ \variance \in (0, \variance_0]}} \, \ClipEmpRisk(\LearnedParams, \variance, \Data) - 1 + \frac{1}{\MinPS} + \frac{ \KLBound(\LoggingParams, \variance_0, \LearnedParams, \variance) + 2 \ln\frac{\DataSize}{\delta} }{ \MinPS \( \DataSize - 1 \) } \\
	&= \argmin_{\substack{\LearnedParams \in \Reals^{\FeatDim} \\ \variance \in (0, \variance_0]}} \, \ClipEmpRisk(\LearnedParams, \variance, \Data) + \frac{ \KLBound(\LoggingParams, \variance_0, \LearnedParams, \variance) }{ \MinPS \( \DataSize - 1 \) } \\
	&= \argmin_{\substack{\LearnedParams \in \Reals^{\FeatDim} \\ \variance \in (0, \variance_0]}} \, \ClipEmpRisk(\LearnedParams, \variance, \Data) + \frac{ \frac{1}{\variance_0} \norm{\LearnedParams - \LoggingParams}^2 + \FeatDim \ln\frac{\variance_0}{\variance} }{ \MinPS \( \DataSize - 1 \) } \\
	&= \argmin_{\substack{\LearnedParams \in \Reals^{\FeatDim} \\ \variance \in (0, \variance_0]}} \, \ClipEmpRisk(\LearnedParams, \variance, \Data) + \frac{ \frac{1}{\variance_0} \norm{\LearnedParams - \LoggingParams}^2 - \FeatDim \ln\variance }{ \MinPS \( \DataSize - 1 \) } ,
\label{eq:optimization_equivalence}
\end{align}
we obtain \cref{eq:mixed_logit_bcrm_objective}.

To prove \cref{prop:mixed_logit_bcrm_convex_objective}, we upper-bound $\ClipEmpRisk(\LearnedParams, \variance, \Data)$ by using the fact that $u \ln v \leq u v$ for $u, v \geq 0$. Setting
\begin{equation}
u_i \defeq \frac{\reward_i}{\max\{ \prob_i, \MinPS \}}
\eqand
v_i \defeq \frac{\SoftmaxPolicy_{\LearnedParams}(\action_i \| \context_i)}{\exp(\frac{\variance \FeatMaxNorm^2}{2})} ,
\end{equation}
we have that
\begin{align}
\ClipEmpRisk(\LearnedParams, \variance, \Data) - 1
	&= -\frac{1}{\DataSize} \sum_{i=1}^{\DataSize} \frac{\reward_i}{\max\{ \prob_i, \MinPS \}} \, \frac{\SoftmaxPolicy_{\LearnedParams}(\action_i \| \context_i)}{\exp(\frac{\variance \FeatMaxNorm^2}{2})} \\
	&= - \frac{1}{\DataSize} \sum_{i=1}^{\DataSize} u_i v_i \\
	&\leq -\frac{1}{\DataSize} \sum_{i=1}^{\DataSize} u_i \ln v_i .
\end{align}
Let
\begin{equation}
\gamma \defeq \frac{1}{\MinPS} - \frac{1}{\DataSize} \sum_{i=1}^{\DataSize} u_i \ln v_i ,
\end{equation}
and observe that $\alpha \leq \gamma$. Thus, by \cref{eq:geometric_mean,eq:geometric_mean_to_arithmetic_mean},
\begin{equation}
\Risk(\Policy_{\Posterior})
	\leq \frac{3(\alpha + \beta)}{2}
	\leq \frac{3(\gamma + \beta)}{2} .
\end{equation}
Optimizing this upper bound yields the following equivalence:
\begin{align}
& \argmin_{\substack{\LearnedParams \in \Reals^{\FeatDim} \\ \variance \in (0, \variance_0]}} \, \frac{3(\gamma + \beta)}{2} \\
	= \, & \argmin_{\substack{\LearnedParams \in \Reals^{\FeatDim} \\ \variance \in (0, \variance_0]}} \, \gamma + \beta \\
	= \, & \argmin_{\substack{\LearnedParams \in \Reals^{\FeatDim} \\ \variance \in (0, \variance_0]}} \, \frac{1}{\MinPS} + \frac{1}{\DataSize} \sum_{i=1}^{\DataSize} - u_i \ln v_i + \frac{ \KLBound(\LoggingParams, \variance_0, \LearnedParams, \variance) + 2 \ln\frac{\DataSize}{\delta} }{ \MinPS \( \DataSize - 1 \) } \\
	= \, & \argmin_{\substack{\LearnedParams \in \Reals^{\FeatDim} \\ \variance \in (0, \variance_0]}} \, \frac{1}{\DataSize} \sum_{i=1}^{\DataSize} - u_i \ln v_i + \frac{\norm{\LearnedParams - \LoggingParams}^2}{\variance_0 \, \MinPS \( \DataSize - 1 \)} - \frac{\FeatDim \ln\variance}{\MinPS \( \DataSize - 1 \)} \\
	= \, & \argmin_{\substack{\LearnedParams \in \Reals^{\FeatDim} \\ \variance \in (0, \variance_0]}} \, \frac{1}{\DataSize} \sum_{i=1}^{\DataSize} - \frac{\reward_i \ln\SoftmaxPolicy_{\LearnedParams}(\action_i \| \context_i)}{\max\{ \prob_i, \MinPS \}} + \frac{\reward_i \variance \FeatMaxNorm^2}{2 \max\{ \prob_i, \MinPS \}} + \frac{\norm{\LearnedParams - \LoggingParams}^2}{\variance_0 \MinPS (\DataSize - 1)} - \frac{\FeatDim \ln\variance}{\MinPS (\DataSize - 1)} .
\end{align}
Observe that $\LearnedParams$ and $\variance$ never interact multiplicatively in the objective function. We can therefore solve each sub-optimization separately.

Starting with $\LearnedParams$, we simply isolate the relevant terms and obtain \cref{eq:mixed_logit_bcrm_convex_objective}. For $\variance$, we must solve
\begin{equation}
\argmin_{\variance \in (0, \variance_0]} \, \frac{1}{\DataSize} \sum_{i=1}^{\DataSize} \frac{\reward_i \FeatMaxNorm^2 \variance}{2 \max\{ \prob_i, \MinPS \}} - \frac{\FeatDim \ln\variance}{\MinPS \( \DataSize - 1 \)} .
\end{equation}
Note that this objective is convex in $\variance$. If we ignore the constraint that $\variance \in (0, \variance_0]$ and let $\variance$ be any real number, then the problem has an analytic solution:
\begin{equation}
\argmin_{\variance \in \Reals} \, \frac{1}{\DataSize} \sum_{i=1}^{\DataSize} \frac{\reward_i \FeatMaxNorm^2 \variance}{2 \max\{ \prob_i, \MinPS \}} - \frac{\FeatDim \ln\variance}{\MinPS (\DataSize - 1)}
	= \frac{2 \FeatDim}{\FeatMaxNorm^2 \MinPS (\DataSize - 1)} \( \frac{1}{\DataSize} \sum_{i=1}^{\DataSize} \frac{\reward_i}{\max\{ \prob_i, \MinPS \}} \)^{-1} .
\end{equation}
This can be verified by setting the derivative equal to 0 and solving for $\variance$. Suppose the solution to the unconstrained problem lies outside of the feasible region for the constrained problem, $(0, \variance_0]$. It is easily verified that the unconstrained solution is strictly positive; thus, it must be greater than $\variance_0$. Since the objective function is convex, we must then have that the solution to the constrained problem lies at the upper boundary, $\variance_0$, which is the closest point to the unconstrained solution. Thus, the minimizer of the constrained problem is either the unconstrained solution or $\variance_0$; whichever one is smaller.

\subsection{Proof of \cref{th:data_dep_reg_risk_bound}}
\label{sec:proof_data_dep_reg_risk_bound}

To prove \cref{th:data_dep_reg_risk_bound}, we start by borrowing a result from \citet{liu:icml17}, which we simplify and specialize for our use case.
\begin{lemma}[{\citep[Lemma 1]{liu:icml17}}]
\label{lem:stability_concentration}
Let $\Hamming(\Data, \Data')$ denote the Hamming distance between two datasets, $\Data, \Data'$. Suppose there exists a constant, $\stability > 0$, such that
\begin{equation}
\sup_{\Data, \Data' : \Hamming(\Data, \Data') = 1}  \norm{\RERM - \RERM[\Data']} \,\leq\, \stability .
\label{eq:stability}
\end{equation}
(In other words, perturbing any single training example can change the learned parameters by at most $\stability$.) Then, for any $\delta \in (0,1)$,
\begin{equation}
\Pr_{\Data \by (\ContextDistribution \times \LoggingPolicy)^{\DataSize}} \left\{ \norm{ \RERM - \MeanRERM } \geq \stability \sqrt{ 2 \DataSize \ln\frac{2}{\delta} } \right\} \leq \delta .
\label{eq:stability_concentration}
\end{equation}
\end{lemma}

To apply \cref{lem:stability_concentration}, we must identify a value of $\stability$ that satisfies \cref{eq:stability}.
\begin{lemma}
\label{lem:rerm_stability}
If the loss function, $\Loss$, is convex and $\lipschitz$-Lipschitz with respect to its first argument, then the minimizer, $\RERM$, satisfies \cref{eq:stability} for $\stability = \frac{\lipschitz}{\regparam \DataSize}$.
\end{lemma}
\begin{proof}
Without loss of generality, assume that the index of the example at which $\Data$ and $\Data'$ differ is $i$. It easily verified that the regularizer, $\regparam \norm{\params}^2$, is $(2\regparam)$-strongly convex; and since $\Loss$ is assumed to be convex, $\RERMObjective$ (\cref{eq:rerm_objective}), is also $(2\regparam)$-strongly convex. Therefore, using the definition of strongly convex functions, and the symmetry of distances, we have that
\begin{align}
\norm{\RERM - \RERM[\Data']}^2 ~
	= ~~~ &\frac{1}{2} \norm{\RERM - \RERM[\Data']}^2 + \frac{1}{2} \norm{\RERM[\Data'] - \RERM}^2 \\
	\leq ~~~ &\frac{1}{2 \regparam} \( \RERMObjective(\RERM, \Data') - \RERMObjective(\RERM[\Data'], \Data') \) \\
	+ ~ &\frac{1}{2 \regparam} \( \RERMObjective(\RERM[\Data'], \Data) - \RERMObjective(\RERM, \Data) \) \\
	= ~~~ &\frac{1}{2 \regparam} \( \RERMObjective(\RERM[\Data'], \Data) - \RERMObjective(\RERM[\Data'], \Data') \) \\
	+ ~ &\frac{1}{2 \regparam} \( \RERMObjective(\RERM, \Data') - \RERMObjective(\RERM, \Data) \) \\
	= ~~~ &\frac{1}{2 \regparam \DataSize} \( \Loss(\RERM[\Data'], \context_i, \action_i) - \Loss(\RERM[\Data'], \context'_i, \action'_i) \) \\
	+ ~ &\frac{1}{2 \regparam \DataSize} \( \Loss(\RERM, \context'_i, \action'_i) - \Loss(\RERM, \context_i, \action_i) \) \\
	= ~~~ &\frac{1}{2 \regparam \DataSize} \( \Loss(\RERM[\Data'], \context_i, \action_i) - \Loss(\RERM, \context_i, \action_i) \) \\
	+ ~ &\frac{1}{2 \regparam \DataSize} \( \Loss(\RERM, \context'_i, \action'_i) - \Loss(\RERM[\Data'], \context'_i, \action'_i) \) \\
	\leq ~~~ &\frac{\lipschitz}{2 \regparam \DataSize} \( \norm{\RERM[\Data'] - \RERM} + \norm{\RERM - \RERM[\Data']} \) \\
	= ~~~ &\frac{\lipschitz}{\regparam \DataSize} \norm{\RERM - \RERM[\Data']} .
\end{align}
Dividing each side by $\norm{\RERM - \RERM[\Data']}$ completes the proof.
\end{proof}

Now, we can apply \cref{lem:stability_concentration} to show that $\RERM$ concentrates around $\MeanRERM$.
\begin{lemma}
\label{lem:rerm_concentration}
If the loss function, $\Loss$, is convex and $\lipschitz$-Lipschitz with respect to its first argument, then for any $\delta \in (0,1)$,
\begin{equation}
\Pr_{\Data \by (\ContextDistribution \times \LoggingPolicy)^{\DataSize}} \left\{ \norm{ \RERM - \MeanRERM } \geq \frac{\lipschitz}{\regparam} \sqrt{ \frac{2 \ln\frac{2}{\delta}}{\DataSize} } \right\} \leq \delta .
\label{eq:rerm_concentration}
\end{equation}
\end{lemma}
\begin{proof}
Follows immediately from \cref{lem:stability_concentration,lem:rerm_stability}, with $\stability = \frac{\lipschitz}{\regparam \DataSize}$.
\end{proof}

We are now ready to prove \cref{th:data_dep_reg_risk_bound}. We start by applying \cref{th:mixed_logit_risk_bound}, with $\LoggingParams$ replaced by $\MeanRERM$, and $\delta$ replaced by $\delta / 2$. With probability at least $1 - \delta / 2$,
\begin{align}
\Risk(\Policy_{\Posterior})
	\leq \ClipEmpRisk(\LearnedParams, \variance, \Data)
	&+ \sqrt{ \frac{ \big( \ClipEmpRisk(\LearnedParams, \variance, \Data) - 1 + \frac{1}{\MinPS} \big) \big( \KLBound(\MeanRERM, \variance_0, \LearnedParams, \variance) + 2 \ln\frac{2\DataSize}{\delta} \big) }{ \MinPS \( \DataSize - 1 \) } } \\
	&+ \frac{ \( \KLBound(\MeanRERM, \variance_0, \LearnedParams, \variance) + 2 \ln\frac{2\DataSize}{\delta} \) }{ \MinPS \( \DataSize - 1 \) } .
\end{align}
Then, using the triangle inequality and \cref{lem:rerm_concentration}, we have that
\begin{align}
\norm{\LearnedParams - \MeanRERM}
	&\leq \norm{\LearnedParams - \RERM} + \norm{ \RERM - \MeanRERM } \\
	&\leq \norm{\LearnedParams - \RERM} + \frac{\lipschitz}{\regparam} \sqrt{ \frac{2 \ln\frac{4}{\delta}}{\DataSize} } ,
\end{align}
with probability at least $1 - \delta / 2$. Substituting this into \cref{eq:mixed_logit_kl_bound} yields
\begin{align}
\KLBound(\MeanRERM, \variance_0, \LearnedParams, \variance)
	&= \frac{\norm{\LearnedParams - \MeanRERM}^2}{\variance_0} + \FeatDim \ln\frac{\variance_0}{\variance} \\
	&\leq \frac{ \( \norm{\LearnedParams - \RERM} + \frac{\lipschitz}{\regparam} \sqrt{ \frac{2 \ln\frac{4}{\delta}}{\DataSize} } \)^2 }{\variance_0} + \FeatDim \ln\frac{\variance_0}{\variance} \\
	&= \hat{\KLBound}(\RERM, \variance_0, \LearnedParams, \variance) ,
\end{align}
with probability at least $1 - \delta / 2$. Thus, \cref{eq:data_dep_reg_risk_bound} holds with probability at least $1 - \delta$.

\bibliographystyle{plainnat}
\bibliography{bcrm}

\begin{thebibliography}{35}
\providecommand{\natexlab}[1]{#1}
\providecommand{\url}[1]{\texttt{#1}}
\expandafter\ifx\csname urlstyle\endcsname\relax
  \providecommand{\doi}[1]{doi: #1}\else
  \providecommand{\doi}{doi: \begingroup \urlstyle{rm}\Url}\fi

\bibitem[Aitchison and Shen(1980)]{aitchison:biometrika80}
J.~Aitchison and S.~Shen.
\newblock Logistic-normal distributions: Some properties and uses.
\newblock \emph{Biometrika}, 67\penalty0 (2):\penalty0 261--272, 1980.

\bibitem[Beygelzimer and Langford(2009)]{beygelzimer:kdd09}
A.~Beygelzimer and J.~Langford.
\newblock The offset tree for learning with partial labels.
\newblock In \emph{Knowledge Discovery and Data Mining}, 2009.

\bibitem[Bottou et~al.(2013)Bottou, Peters, {n}onero Candela, Charles,
  Chickering, Portugaly, Ray, Simard, and Snelson]{bottou:jmlr13}
L.~Bottou, J.~Peters, J.~Qui\ {n}onero Candela, D.~Charles, D.~Chickering,
  E.~Portugaly, D.~Ray, P.~Simard, and E.~Snelson.
\newblock Counterfactual reasoning and learning systems: The example of
  computational advertising.
\newblock \emph{Journal of Machine Learning Research}, 14:\penalty0 3207--3260,
  2013.

\bibitem[Catoni(2007)]{catoni:ims07}
O.~Catoni.
\newblock \emph{{PAC}-Bayesian Supervised Classification: The Thermodynamics of
  Statistical Learning}, volume~56 of \emph{Institute of Mathematical
  Statistics Lecture Notes -- Monograph Series}.
\newblock Institute of Mathematical Statistics, 2007.

\bibitem[Chapelle and Li(2011)]{chapelle:nips11}
O.~Chapelle and L.~Li.
\newblock An empirical evaluation of {T}hompson sampling.
\newblock In \emph{Neural Information Processing Systems}, 2011.

\bibitem[Deng et~al.(2009)Deng, Dong, Socher, Li, Li, and Li]{deng:cvpr09}
J.~Deng, W.~Dong, R.~Socher, L.-J. Li, K.~Li, and F.-F. Li.
\newblock {ImageNet}: A large-scale hierarchical image database.
\newblock In \emph{IEEE Conference on Computer Vision and Pattern Recognition},
  2009.

\bibitem[Duchi et~al.(2011)Duchi, Hazan, and Singer]{duchi:jmlr11}
J.~Duchi, E.~Hazan, and Y.~Singer.
\newblock Adaptive subgradient methods for online learning and stochastic
  optimization.
\newblock \emph{Journal of Machine Learning Research}, 12:\penalty0 2121--2159,
  2011.

\bibitem[Dud\'{i}k et~al.(2011)Dud\'{i}k, Langford, and Li]{dudik:icml11}
M.~Dud\'{i}k, J.~Langford, and L.~Li.
\newblock Doubly robust policy evaluation and learning.
\newblock In \emph{International Conference on Machine Learning}, 2011.

\bibitem[Dziugaite and Roy(2018)]{dziugaite:nips18}
G.~Dziugaite and D.~Roy.
\newblock Data-dependent {PAC}-{B}ayes priors via differential privacy.
\newblock In \emph{Neural Information Processing Systems}, 2018.

\bibitem[Germain et~al.(2009)Germain, Lacasse, Laviolette, and
  Marchand]{germain:icml09}
P.~Germain, A.~Lacasse, F.~Laviolette, and M.~Marchand.
\newblock {PAC}-{B}ayesian learning of linear classifiers.
\newblock In \emph{International Conference on Machine Learning}, 2009.

\bibitem[He et~al.(2016)He, Zhang, Ren, and Sun]{he:cvpr16}
K.~He, X.~Zhang, S.~Ren, and J.~Sun.
\newblock Deep residual learning for image recognition.
\newblock In \emph{IEEE Conference on Computer Vision and Pattern Recognition},
  2016.

\bibitem[Horvitz and Thompson(1952)]{horvitz:jasa52}
D.~Horvitz and D.~Thompson.
\newblock A generalization of sampling without replacement from a finite
  universe.
\newblock \emph{Journal of the American Statistical Association}, 47\penalty0
  (260):\penalty0 663--685, 1952.

\bibitem[Ionides(2008)]{ionides:jcgs08}
E.~Ionides.
\newblock Truncated importance sampling.
\newblock \emph{Journal of Computational and Graphical Statistics}, 17\penalty0
  (2):\penalty0 295--311, 2008.

\bibitem[Joachims et~al.(2018)Joachims, Swaminathan, and
  de~Rijke]{joachims:iclr18}
T.~Joachims, A.~Swaminathan, and M.~de~Rijke.
\newblock Deep learning with logged bandit feedback.
\newblock In \emph{International Conference on Learning Representations}, 2018.

\bibitem[Krizhevsky and Hinton(2009)]{krizhevsky:tech09}
A.~Krizhevsky and G.~Hinton.
\newblock Learning multiple layers of features from tiny images.
\newblock Technical report, University of Toronto, 2009.

\bibitem[Langford and Shawe-Taylor(2002)]{langford:nips02}
J.~Langford and J.~Shawe-Taylor.
\newblock {PAC}-{B}ayes and margins.
\newblock In \emph{Neural Information Processing Systems}, 2002.

\bibitem[Lever et~al.(2010)Lever, Laviolette, and Shawe-Taylor]{lever:alt10}
G.~Lever, F.~Laviolette, and J.~Shawe-Taylor.
\newblock Distribution-dependent {PAC}-{B}ayes priors.
\newblock In \emph{Algorithmic Learning Theory}, 2010.

\bibitem[Liu et~al.(2017)Liu, Lugosi, Neu, and Tao]{liu:icml17}
T.~Liu, G.~Lugosi, G.~Neu, and D.~Tao.
\newblock Algorithmic stability and hypothesis complexity.
\newblock In \emph{International Conference on Machine Learning}, 2017.

\bibitem[Ma et~al.(2019)Ma, Wang, and Narayanaswamy]{ma:aistats19}
Y.~Ma, Y.-X. Wang, and B.~Narayanaswamy.
\newblock Imitation-regularized offline learning.
\newblock In \emph{Artificial Intelligence and Statistics}, 2019.

\bibitem[McAllester(1999)]{mcallester:colt99}
D.~McAllester.
\newblock {PAC}-{B}ayesian model averaging.
\newblock In \emph{Computational Learning Theory}, 1999.

\bibitem[McAllester(2003)]{mcallester:colt03}
D.~McAllester.
\newblock Simplified {PAC}-{B}ayesian margin bounds.
\newblock In \emph{COLT}, pages 203--215, 2003.

\bibitem[Mohri et~al.(2012)Mohri, Rostamizadeh, and Talwalkar]{mohri:book12}
M.~Mohri, A.~Rostamizadeh, and A.~Talwalkar.
\newblock \emph{Foundations of Machine Learning}.
\newblock Adaptive computation and machine learning. {MIT} Press, 2012.
\newblock ISBN 978-0-262-01825-8.

\bibitem[Parrado-Hern\'{a}ndez et~al.(2012)Parrado-Hern\'{a}ndez, Ambroladze,
  Shawe-Taylor, and Sun]{hernandez:jmlr12}
E.~Parrado-Hern\'{a}ndez, A.~Ambroladze, J.~Shawe-Taylor, and S.~Sun.
\newblock {PAC}-{B}ayes bounds with data dependent priors.
\newblock \emph{Journal of Machine Learning Research}, 13:\penalty0 3507--3531,
  2012.

\bibitem[Rivasplata et~al.(2018)Rivasplata, Parrado-Hern\'{a}ndez,
  Shawe-Taylor, Sun, and Szepesv\'{a}ri]{rivasplata:nips18}
O.~Rivasplata, E.~Parrado-Hern\'{a}ndez, J.~Shawe-Taylor, S.~Sun, and
  C.~Szepesv\'{a}ri.
\newblock {PAC}-{B}ayes bounds for stable algorithms with instance-dependent
  priors.
\newblock In \emph{Neural Information Processing Systems}, 2018.

\bibitem[Rosenbaum and Rubin(1983)]{rosenbaum:biometrika83}
P.~Rosenbaum and D.~Rubin.
\newblock The central role of the propensity score in observational studies for
  causal effects.
\newblock \emph{Biometrika}, 70:\penalty0 41--55, 1983.

\bibitem[Schulman et~al.(2015)Schulman, Levine, Abbeel, Jordan, and
  Moritz]{schulman:icml15}
J.~Schulman, S.~Levine, P.~Abbeel, M.~Jordan, and P.~Moritz.
\newblock Trust region policy optimization.
\newblock In \emph{International Conference on Machine Learning}, 2015.

\bibitem[Seeger(2002)]{seeger:jmlr02}
M.~Seeger.
\newblock {PAC}-{B}ayesian generalisation error bounds for {G}aussian process
  classification.
\newblock \emph{Journal of Machine Learning Research}, 3:\penalty0 233--269,
  2002.

\bibitem[Seldin et~al.(2011)Seldin, Auer, Laviolette, Shawe-Taylor, and
  Ortner]{seldin:nips11}
Y.~Seldin, P.~Auer, F.~Laviolette, J.~Shawe-Taylor, and R.~Ortner.
\newblock {PAC}-{B}ayesian analysis of contextual bandits.
\newblock In \emph{Neural Information Processing Systems}, 2011.

\bibitem[Seldin et~al.(2012)Seldin, Laviolette, Cesa-Bianchi, Shawe-Taylor, and
  Auer]{seldin:it12}
Y.~Seldin, F.~Laviolette, N.~Cesa-Bianchi, J.~Shawe-Taylor, and Peter Auer.
\newblock {PAC}-{B}ayesian inequalities for martingales.
\newblock \emph{IEEE Transactions on Information Theory}, 58\penalty0
  (12):\penalty0 7086--7093, 2012.

\bibitem[Strehl et~al.(2010)Strehl, Langford, Li, and Kakade]{strehl:nips10}
A.~Strehl, J.~Langford, L.~Li, and S.~Kakade.
\newblock Learning from logged implicit exploration data.
\newblock In \emph{Neural Information Processing Systems}, 2010.

\bibitem[Sutton et~al.(2000)Sutton, McAllester, Singh, and
  Mansour]{sutton:nips00}
R.~Sutton, D.~McAllester, S.~Singh, and Y.~Mansour.
\newblock Policy gradient methods for reinforcement learning with function
  approximation.
\newblock In \emph{Neural Information Processing Systems}, 2000.

\bibitem[Swaminathan and Joachims(2015{\natexlab{a}})]{swaminathan:jmlr15}
A.~Swaminathan and T.~Joachims.
\newblock Batch learning from logged bandit feedback through counterfactual
  risk minimization.
\newblock \emph{Journal of Machine Learning Research}, 16:\penalty0 1731--1755,
  2015{\natexlab{a}}.

\bibitem[Swaminathan and Joachims(2015{\natexlab{b}})]{swaminathan:nips15}
A.~Swaminathan and T.~Joachims.
\newblock The self-normalized estimator for counterfactual learning.
\newblock In \emph{Neural Information Processing Systems}, 2015{\natexlab{b}}.

\bibitem[Tolstikhin and Seldin(2013)]{tolstikhin:nips13}
Ilya Tolstikhin and Yevgeny Seldin.
\newblock {PAC}-{B}ayes-empirical-{B}ernstein inequality.
\newblock In \emph{Neural Information Processing Systems}, 2013.

\bibitem[Xiao et~al.(2017)Xiao, Rasul, and Vollgraf]{xiao:corr17}
H.~Xiao, K.~Rasul, and R.~Vollgraf.
\newblock Fashion-{MNIST}: a novel image dataset for benchmarking machine
  learning algorithms.
\newblock \emph{CoRR}, abs/1708.07747, 2017.

\end{thebibliography}

\end{document}